\documentclass{article}
\usepackage[utf8]{inputenc}
\usepackage[margin=1in]{geometry}

\usepackage{microtype}
\usepackage{graphicx}
\usepackage{subfigure}
\usepackage{booktabs}
\usepackage{natbib}
\setcitestyle{open={(},close={)}}
\newcommand{\independent}{\rotatebox[origin=c]{90}{$\models$}}

\usepackage[colorlinks=true,citecolor=blue]{hyperref}

\usepackage{amsmath,amsthm,amsfonts,amssymb,mathdots,array,mathrsfs,bm,bbm,stmaryrd,graphicx,subfigure,xcolor}
\usepackage[vlined,linesnumbered,ruled,resetcount]{algorithm2e}

\usepackage{breakcites}

\usepackage[T1]{fontenc}
\usepackage{enumerate}
\usepackage{inputenc}

\usepackage{graphicx} 
\usepackage{subfigure}

\usepackage{booktabs,balance}
\usepackage{rotating}
\usepackage{boldline}
\usepackage{makecell}
\usepackage{multirow}
\usepackage{balance}

\usepackage{tikz}

\newtheorem{theorem}{Theorem}[section]

\newtheorem{lemma}[theorem]{Lemma}

\newtheorem{remark}{Remark}[section]

\newcommand{\bsmat}{\begin{bmatrix} }
\newcommand{\esmat}{\end{bmatrix} }

\usepackage{environ}
\NewEnviron{smallequation}{%
    \begin{equation}
    \scalebox{0.97}{$\BODY$}
    \end{equation}
    }
    \NewEnviron{smallalign}{%
    \begin{equation}
    \scalebox{0.97}{$\BODY$}
    \end{equation}
    }

\begin{document}

\title{\bf C-OPH: Improving the Accuracy of One Permutation Hashing (OPH) with Circulant Permutations}

\author{\textbf{Xiaoyun Li and Ping Li} \\\\
Cognitive Computing Lab\\
Baidu Research\\
10900 NE 8th St. Bellevue, WA 98004, USA\\
  \texttt{\{xiaoyunli,\ liping11\}@baidu.com}
}
\date{\vspace{0.1in}}
\maketitle

\begin{abstract}
\noindent Minwise hashing (MinHash) is a classical method for efficiently estimating the Jaccrad similarity in massive binary (0/1) data. To generate $K$ hash values for each data vector, the standard theory of MinHash requires $K$ independent permutations. Interestingly, the recent work on ``circulant MinHash'' (C-MinHash)~\citep{CMH2Perm2021} has shown that merely two permutations are needed. The first permutation breaks the structure of the data and the second permutation is re-used $K$ time in a circulant manner. Surprisingly, the estimation accuracy of C-MinHash is proved to be strictly smaller than that of the original MinHash.  The more recent work~\citep{CMH1Perm2021} further demonstrates that practically only one permutation is needed. Note that those two papers~\citep{CMH1Perm2021,CMH2Perm2021} are different from the well-known work on ``One Permutation Hashing (OPH)'' published in NIPS'12~\citep{Proc:Li_Owen_Zhang_NIPS12}. OPH and its variants using different ``densification'' schemes are popular alternatives to the standard MinHash. The densification step is necessary in order to deal with empty bins which exist in One Permutation~Hashing.\\

\noindent In this paper, we propose to incorporate the essential ideas of C-MinHash to improve the accuracy of One Permutation Hashing. Basically, we develop a new densification method for OPH, which achieves the  smallest estimation variance compared to all existing densification schemes for OPH. Our proposed method is named C-OPH (Circulant OPH). After the initial permutation (which breaks the existing structure of the data), C-OPH only needs a ``shorter'' permutation of length $D/K$ (instead of $D$), where $D$ is the original data dimension and $K$ is the total number of bins in OPH. This short permutation is re-used in $K$ bins in a circulant shifting manner. It can be shown that the estimation variance of the Jaccard similarity is strictly smaller than that of the existing (densified) OPH methods. \\

\noindent Additionally, our work leads to an interesting and useful consequence. That is, if we neglect the cost of the initial permutation needed by the bin-dividing step, then we actually just need ``$1/K$'' permutation instead of one permutation. It turns out that the initial permutation can be safely replaced by an approximate permutation such as 2-universal (2U) hashing, as verified by an extensive empirical study.
\end{abstract}

\newpage

\section{Introduction}  \label{sec:intro}

With the explosive growth in the scale of data, efficient methods for search, data storage, and large-scale machine learning have become increasingly important. The method of minwise hashing (MinHash)~\citep{Proc:Broder,Proc:Broder_STOC98,Article:Li_Konig_CACM11} is a classical and popular hashing algorithm for binary 0/1 data. The hashed values  of two data points $\bm v,\bm w\in\{0,1\}^D$, which are non-negative integers, have collision probability equal to the Jaccard similarity (resemblance) between the two data points defined as
\begin{align} \label{def:J}
    J(\bm v,\bm w)=\frac{\sum_{i=1}^D \mathbbm 1\{v_i=w_i=1\}}{\sum_{i=1}^D \mathbbm 1\{v_i+w_i\geq 1\}}.
\end{align}

When using MinHash in practice,  in order to ensure reliable accuracy, one typically needs to generate hundreds to thousands hash values per data point depending on the applications.
We use $K$ to denote the target number of hashes (per data point). Operationally, to achieve a strict performance guarantee that matches the theory, the standard MinHash requires $K$ independent permutations, with each hash value being the minimal permuted index of the non-zero elements of the data vector.

Over the past decades, MinHash and variants have been widely used in numerous applications,  for similarity estimation, approximate nearest neighbor search, duplicate detection, clustering, classification, image retrieval, database systems,  etc.~\citep{Proc:Broder_WWW97,Proc:Charikar_STOC02,Proc:Fetterly_WWW03,Proc:google_WWW07,Proc:Buehrer_WSDM08,Proc:Bendersky_WSDM09,Proc:Lee_ECCV10,Proc:Li_NIPS11,Proc:Deng_CIKM12,Proc:Chum_CVPR12,Proc:Shrivastava_ECML12,Proc:malware_KDD14,Proc:Zhu_VLDB17,Proc:Nargesian_VLDB18,Proc:Wang_KDD19,Proc:Lemiesz_VLDB21,Proc:Tseng_SIGMOD21,Proc:Feng_SIGMOD21,Proc:Jia_SIGMOD21}.

Here, we should also mention that, in the very early development of minwise hashing,  only one permutation was used by storing the first $K$ non-zero locations after the permutation~\citep{Proc:Broder,Proc:Broder_WWW97}. Later \citet{Proc:Li_Church_EMNLP05} proposed maximum likelihood estimators to significantly improve the estimation accuracy, and  \citet{Article:Li_Church_CL07} further extended the method to estimating three-way and multi-way associations. However,
because the hashed values in~\citet{Proc:Broder,Proc:Broder_WWW97,Proc:Li_Church_EMNLP05,Article:Li_Church_CL07}  did not form a metric space (i.e., satisfy the triangle inequality), they could not be used for numerous important applications that require metric space.

\subsection{Circulant Minwise Hashing (C-MinHash)}

Recently, \cite{CMH2Perm2021} proposes a convenient variant of MinHash that rigorously reduces the number of independent permutations required to merely 2. In the so-called Circulant MinHash method, noted as C-MinHash-$(\sigma,\pi)$, the first permutation $\sigma$ is used for initially permuting the data (pre-processing), and the second (independent) permutation $\pi$ is used for generating the hash values, repeatedly for $K$ times by circulant shifts. For example, if $\pi=[3,2,4,1]$, then $\pi_{\rightarrow 1}=[1,3,2,4]$ is the permutation shifted rightwards by 1 element, and so on. C-MinHash uses $\pi_{\rightarrow 1}$, $\pi_{\rightarrow 2}$, ..., to generate the hash values, instead of the independent permutations as in the standard MinHash. The concrete procedure is summarized in Algorithm~\ref{alg:C-MinHash}.

\begin{algorithm}[h]
{
	\textbf{Input:} Binary data vector $\bm v\in\{0,1\}^D$, \hspace{0.1in} Permutation vectors $\pi$ and $\sigma$: $[D]\rightarrow[D]$

	\textbf{Output:} Hash values $h_1(\bm v),...,h_K(\bm v)$
	
	\vspace{0.05in}

    Initial permutation: $\bm v'$ = $\sigma(\bm v)$

	For $k=1$ to $K$

	\hspace{0.2in}Shift $\pi$ circulantly rightwards by $k$ units: $\pi_k=\pi_{\rightarrow k}$

	\hspace{0.2in}$h_k(\bm v)\leftarrow \min_{i:v_i'\neq 0} \pi_{\rightarrow k}(i)$

	End For

	}\caption{C-MinHash-$(\sigma,\pi)$}
	\label{alg:C-MinHash}
\end{algorithm}

Surprisingly, it was proved in~\citet{CMH2Perm2021} that using merely 2 permutations, the Jaccard estimator of C-MinHash has uniformly smaller variance than that of the classical MinHash (with $K$ independent permutations). Moreover, in a followup work~\citep{CMH1Perm2021}, they further verified by extensive empirical results that actually one can more conveniently use merely one permutation, i.e., simply letting $\sigma = \pi$ in Algorithm~\ref{alg:C-MinHash}. That is, one can use $\pi$ for both initialization (pre-processing) and circulant hashing.  While the dependence between initial permutation and those for hashing makes the theoretical analysis very complicated, the authors of~\citep{CMH1Perm2021} provided the exact expression of the mean estimation of C-MinHash-$(\pi,\pi)$. Even though the Jaccard estimator of C-MinHash-$(\pi,\pi)$ is biased, the bias is too small to have any noticeable impact on the overall mean square error (MSE), where MSE = variance + bias$^2$.

Practically speaking, C-MinHash provides a convenient strategy for the design of hashing systems. For instance, when the data space is not ultra large (e.g., $\leq 2^{30}$, a billion features), we can simply save one ``permutation vector'' to generate all $K$ hash values. A vector of length $2^{30}$ can  be easily stored, even in GPU memory. However, saving thousands of such permutations  might be still unrealistic and wasteful. The benefit of using exact permutations, instead of approximation by hash functions, is that the empirical performance would always match the theory. We will elaborate more on this point later in the paper.

\subsection{One Permutation Hashing (OPH)} \label{intro:OPH}

While C-MinHash methods need to use only one or two permutations, it should be clear that they are different from the work on ``one permutation hashing (OPH)'' and its many variants of ``densification'' schemes. The idea of OPH~\citep{Proc:Li_Owen_Zhang_NIPS12} is to randomly divide the data vector of size $D$ into $K$ equal-sized bins (by applying a random permutation $\sigma$). Then, minwise hashing is applied in each bin, producing $K$ hash values from $K$ bins. An illustrative example is presented in Figure~\ref{fig:OPH example}. The theoretical correctness of this approach is built upon the proof that, the expected ``bin-wise'' Jaccard similarity in each simultaneously non-empty bin is the same as the true $J$ of the whole data vectors $\bm v,\bm w$~\citep{Proc:Li_Owen_Zhang_NIPS12}.

\begin{figure}[h]
\centering
	\includegraphics[width=4in]{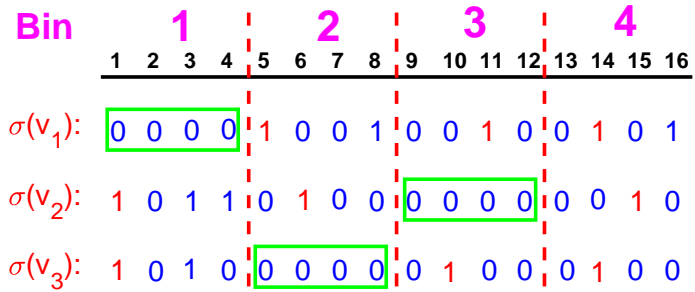}
	\caption{A toy example of One Permutation Hashing (OPH) provided by ~\cite{Proc:Li_NIPS19_BCWS} on three binary vectors, $v_1,v_2,v_3$. The data vector $\bm v$ is first randomly split into $K=4$ bins, $\mathcal B_1, ..., \mathcal B_4$, by a random permutation $\sigma$. Then, the hash value is taken as the smallest non-zero index within each bin (effectively doing a minwise hashing within the bin). If a bin is empty, we record ``$E$''. For instance, $h(v_1)=[E,5,11,14]$.}
	\label{fig:OPH example}
\end{figure}

Note that,  as shown in Figure~\ref{fig:OPH example}, OPH really only needs one permutation (if the number of hash values is smaller than $K$). After the data vector is permuted by $\sigma$ and broken into $K$ bins, the  ``small permutation'' within each bin is already completed. Nevertheless, for the convenience of explaining our proposed idea in this paper, we can still conceptually view OPH as a ``two-permutation'' scheme. That is, the initial permutation $\sigma$ is applied and the data vector is divided after the permutation. A second permutation $\pi$ is broken evenly into $K$ small permutations and each small permutation is used in one bin to conduct minwise hashing within the bin. As proved in~\citet{Proc:Li_Owen_Zhang_NIPS12}, for OPH, we can simply let $\pi=\sigma$.  This interpretation would be helpful to understand the densification and re-randomization procedures which will soon be introduced.

There  are possibly empty bins, especially for relatively sparse data vectors, as reflected in Figure~\ref{fig:OPH example}.   To tackle the problem of empty bins, \citet{Proc:Shrivastava_ICML14,Proc:Shrivastava_UAI14} developed two ``densification'' schemes, which were later improved by~\citet{Proc:Shrivastava_ICML17,Proc:Li_NIPS19_BCWS}, among other works.

\newpage

\subsection{Summary of Contributions}

In this work, we seek to improve the estimation accuracy of the existing OPH framework as well its densification methods, by incorporating the essential idea of circulant MinHash~\citep{CMH2Perm2021,CMH1Perm2021}. After the initial bin-dividing process, we use a ``small'' permutation of size $D/K$, where $D$ is the data dimension and $K$ is the number of bins. Without loss of generality, we  assume $D/K$ is an integer (otherwise we can always pad zeros). This permutation is re-used in all bins in a circulant shifting manner. We derive the precise variance formula of the Jaccard estimator, and show that C-OPH reduces the prior known optimal Jaccard estimation variance of densified OPH~\citep{Proc:Li_NIPS19_BCWS}. Numerical results are provided to validate the theory.

In addition to achieving improved accuracy, there is also another beneficial consequence. That is, if we neglect the cost of the initial permutation for the bin-dividing procedure, we actually only need ``$1/K$ permutation'' to generate $K$ hash values, instead of using one permutation of size $D$.  It turns out that the initial permutation can be safely replaced by an approximation such as the 2-universal (2U) hashing and other approximate hashing methods such as 4-universal (4U) hashing. Consider the original Circulant MinHash with two permutations C-MinHash-$(\sigma,\pi)$. If we replace the initial permutation $\sigma$ by 2U hashing, it becomes C-MinHash-$(2U,\pi)$. Likewise, if we replace the second permutation $\pi$ by 2U hashing, we obtain C-MinHash-$(\sigma,2U)$. Our experimental study has shown that  C-MinHash-$(2U,\pi)$ achieves essentially the same accuracy as C-MinHash-$(\sigma,\pi)$. The accuracy of C-MinHash-$(\sigma,2U)$, however, can be much worse than C-MinHash-$(\sigma,\pi)$. in certain datasets.

The above observation provides the intuition for developing C-OPH-$(2U,\pi/K)$. That is, in Circulant OPH, we use the 2U hashing for the initial permutation and use one single (small) permutation of size $D/K$ to generate hashes from the bins, in a circulant shifting fashion. This strategy would allow practitioners to be able to use C-OPH-$(2U,\pi/K)$ in datasets of much higher dimensions.  For example, consider a dataset with $D = 2^{40}$ and we let $K = 2^{10}$. Then we just need a small permutation of size $D/K = 2^{30}$ for this very high dimensional dataset, and $2^{30}$ is small enough  even for the GPU memory.

\section{Background: Densified One Permutation Hashing (OPH)} \label{sec:pre}

We first provide the details about OPH and its densification schemes. The generic framework of OPH is provided in Algorithm~\ref{alg:OPH}, where $[K]$ denotes the set $\{1,...,K\}$. For illustration, we present ReDen (Re-randomized Densified), the most recent and theoretically the most accurate densified OPH~\citep{Proc:Li_NIPS19_BCWS}. Here, we consider a flexible implementation of OPH to allow the number of bins $K$ and the number of hash values $M$ to be different. This is a convenient and practical setting where one may hope to continue generating more than $K$ hashes after each bin has contributed one hash.

\begin{algorithm}[tb]
{   \setcounter{AlgoLine}{0}
	\textbf{Input:} Binary data vector $\bm v\in\{0,1\}^D$; Number of bins $K$; Number of hash values $M$
	
	\textbf{Output:} OPH hashes $h_1(\bm v),...,h_M(\bm v)$
	
	\vspace{0.05in}
	
	Use a permutation $\sigma$ to randomly split $[D]$ into $K$ equal-size bins $\mathcal B_1,...,\mathcal B_K$, $\mathcal B_i=\{j:(i-1)\frac{D}{K}+1\leq \sigma(j)\leq i\frac{D}{K}\}$
	
	Assign each hash an independent permutation $\pi^{(i)}:[D/K]\mapsto [D/K]$, $i\in [M]$
	
	For $k=1$ to $M$
	
	\hspace{0.2in}{\color{red} Select a bin $i_k\in [K]$ by some strategy (option (i) or (ii)) \hfill //First scan bin selection}
	
	\hspace{0.2in}If $\mathcal B_{i_k}$ is not empty
	
	\hspace{0.4in}$h_k(\bm v)\leftarrow \min_{j\in\mathcal B_{i_k}:v_j\neq 0} \pi^{(k)}(\sigma(j)-(i_k-1)\frac{D}{K})+(i_k-1)\frac{D}{K}$ \hfill //MinHash within non-empty bin
	
	\hspace{0.2in}Else
	
	\hspace{0.4in}$h_k(\bm v)=E$  \hfill //Empty bin
	
	\hspace{0.2in}End If
	
	End For
	
	For $k=1$ to $M$
	
	\hspace{0.2in}If $h_k(\bm v)=E$
	
	\hspace{0.4in}{\color{brown}Select a non-empty bin $i_{k}'\in [K]$ \hfill  //Densification bin selection}
	
	\hspace{0.4in}{\color{blue} $h_k(\bm v)\leftarrow \min_{j\in\mathcal B_{i_{k}'}:v_j\neq 0} \pi^{(k)}(\sigma(j)-(i_k'-1)\frac{D}{K})+(i_{k}'-1)\frac{D}{K}$ \hfill  //Re-randomized Densification}
	
	\hspace{0.2in}End If
	
	End For
}
\caption{Generic One Permutation Hashing (OPH) with Re-randomized Densification (ReDen)}
\label{alg:OPH}
\end{algorithm}

The general procedure of ReDen is as follows:

\begin{enumerate}
    \item \textbf{Bin split and permutation generation:} use a permutation $\sigma$ to randomly split the data vector $\bm v$ into $K$ bins, where $\mathcal B_i=\{j:(i-1)\frac{D}{K}+1\leq \sigma(j)\leq i\frac{D}{K}\}$. Assign each hash $k=1, ..., M$ an independent ``bin-wise'' permutation $\pi^{(k)}:[D/K]\mapsto [D/K]$.

    \item \textbf{First scan:} for $k=1,...,M$, select a bin $i_k\in [K]$ by some strategy. If bin $i_k$ is non-empty, set $h_k(\bm v)=\displaystyle{\min_{j:j\in \mathcal B_{i_k},v_j\neq 0}\pi^{(k)}(\sigma(j)-(i_k-1)\frac{D}{K}})+(i_k-1)\frac{D}{K}$, where the addition is to recover the original index to prevent accidental collision; otherwise, set $h_k(\bm v)=E$.

    \item \textbf{Second scan and Re-randomized Densification:} We do another screening over hash samples $h_k$, $k=1,..., M$: if the $k$-th hash is empty (``$E$''), find a non-empty bin $\mathcal B_{i_{k}'}$. Densify the $k$-th hash by $h_k(\bm v)=\displaystyle{\min_{j:j\in \mathcal B_{i_{k}'},v_j\neq 0}\pi^{(k)}(\sigma(j)-(i_k'-1)\frac{D}{K})+(i_k'-1)\frac{D}{K}}$.
\end{enumerate}

Note that, in the densification process, ReDen applies independent permutation $\pi^{(k)}\independent \pi^{(k')}$ to get the hash sample for an empty bin $k$. That said, the ``local'' permutations used within bins for producing the hash for each $h_k$ are all independent. Moreover, as mentioned in Section~\ref{intro:OPH}, the permutation $\sigma$ used for random bin split also implies the independent bin-wise permutations $\pi^{(1)},..., \pi^{(M)}$ as required, when $M\leq K$. This is because, for the permutation $\sigma$ on $[D]$, its ``partial'' sub-permutations split by consecutive bins (local order) are also perfectly random and independent. Therefore, for the $k$-th hash, applying $\sigma$ for bin split implicitly accomplishes the role of applying an additional bin-wise permutation $\pi^{(k)}$. Consequently, in practical implementation, when $M\leq K$, we only need one long permutation $\sigma$, which achieves bin split and implies all the bin-wise short permutations. When $M>K$, we will need additional $(M-K)$ ``bin-wise'' permutations for the purpose of re-randomization.

\subsection{Bin Selection Strategies}

In Algorithm~\ref{alg:OPH}, two important ingredients are the first scan bin selection and the densification bin selection, as depicted by line 6 and line 15, respectively.  Regarding the first scan bin selection (line 6), there are two available strategies:
\begin{enumerate}[(i)]
    \item We strictly pick the $K$ bins one by one. For hash value $h_k$ with $k>K$, we apply a rotation strategy, i.e., to pick the $[mod(k-1,K)+1]$-th bin.

    \item For every $k\in [M]$, we uniformly randomly choose a bin out of $K$ bins.
\end{enumerate}
Strategy (i) can be regarded as fixing the number of times each bin is chosen, while strategy (ii) is more flexible. Strategy (i) typically has smaller Jaccard estimation variance than (ii)~\citep{Proc:Li_NIPS19_BCWS}. Therefore, in the remainder, we will stick to strategy (i) as the scanning bin selection for all the analysis and simulations.

For the densification bin selection strategy (line 15), the earlier work proposed to rotationally select a nearest non-empty bin in a clock-wise direction. Later, \cite{Proc:Shrivastava_ICML17} proposed to uniformly randomly select non-empty bins for densification through 2-universal hashing, which is also adopted in~\cite{Proc:Li_NIPS19_BCWS}. The exact implementation is more complicated in practice, and we refer interested readers to the aforementioned papers for more details. It is known that the uniform sampling strategy would lead to reduced Jaccard estimation variance. Therefore, in our paper, we will focus on this densification bin selection routine. Note that, the main contribution of this work, which is the improvement brought by using circulation in OPH, is valid for all types or combinations of bin selection methods.

\subsection{Densification Strategies}

Line 16 of Algorithm~\ref{alg:OPH} is another important component of OPH, which is the strategy to generate the (densified) hash value once the non-empty bin is selected. In ReDen, the re-randomization procedure proposes to re-do an independent bin-wise permutation. For example, if the hash $h_k$ is empty (i.e., $h_k(\bm v)=E$) and the $j$-th bin is chosen for densifying $h_k$, then $\pi^{(k)}$ is applied to the data in bin $\mathcal B_j$ to get $h_k(\bm v)$. In earlier works, e.g.,~\cite{Proc:Shrivastava_ICML14,Proc:Shrivastava_ICML17}, for densification, the hash value is directly copied from the non-empty bin to the empty bin, i.e., $h_k(\bm v)\leftarrow h_j(\bm v)$. However, this strategy leads to larger estimation variance especially on highly sparse data, since the densification of many empty bins would output many identical hash values. Intuitively speaking, empty bins would not provide much useful additional information of the data. Overall, the ReDen method as presented in Algorithm~\ref{alg:OPH} with option (i) for bin selection and re-randomized densification is the most recent framework that achieves the smallest Jaccard similarity estimation variance among all densified OPH schemes.

\section{C-OPH: Circulant One Permutation Hashing}

In this section, we present our proposed C-OPH method, inspired by the idea of circulant minwise hashing.

\begin{algorithm}[b!]
{   \setcounter{AlgoLine}{0}
	\textbf{Input:} Binary data vector $\bm v\in\{0,1\}^D$; Number of bins $K$; Number of hash values $M$
	
	\hspace{0.5in}Permutation $\sigma: [D]\mapsto [D]$, $\pi: [D/K]\mapsto [D/K]$
	
	\textbf{Output:} C-OPH hashes $h_1(\bm v),...,h_M(\bm v)$
	
	\vspace{0.05in}
	
	Use $\sigma$ to randomly split $[D]$ into $K$ equal-size bins $\mathcal B_1,...,\mathcal B_K$, $\mathcal B_i=\{j:(i-1)\frac{D}{K}+1\leq \sigma(j)\leq i\frac{D}{K}\}$

	For $k=1$ to $M$
	
    \hspace{0.2in}{\color{red} Select a bin $i_k=mod(k-1,K)+1$ by option (i)}

	\hspace{0.2in}If $\mathcal B_{i_k}$ is not empty
	
	\hspace{0.4in}$h_k(\bm v)\leftarrow \min_{j\in\mathcal B_{i_k}:v_j\neq 0} \pi_{\rightarrow k}(\sigma(j)-(i_k-1)\frac{D}{K})+(i_k-1)\frac{D}{K}$ \hfill //MinHash within non-empty bin

	\hspace{0.2in}Else
	
	\hspace{0.4in}$h_k(\bm v)=E$  \hfill //Empty bin
	
	\hspace{0.2in}End If
	
	End For
	
	For $k=1$ to $M$
	
	\hspace{0.2in}If $h_k(\bm v)=E$
	
	\hspace{0.4in}{\color{brown}Select a non-empty bin $i_{k}'\in [M]$ uniformly at random}
	
	\hspace{0.4in}{\color{blue} $h_k(\bm v)\leftarrow \min_{j\in\mathcal B_{i_{k}'}:v_j\neq 0} \pi_{\rightarrow k}(\sigma(j)-(i_k'-1)\frac{D}{K})+(i_k'-1)\frac{D}{K}$ \hfill  //C-OPH Densification}
	
	\hspace{0.2in}End If
	
	End For
}
\caption{Circulant One Permutation Hashing (C-OPH-$(\sigma,\pi)$)}
\label{alg:C-OPH}
\end{algorithm}

\begin{figure}[tb]
\centering
	\includegraphics[width=4.7in]{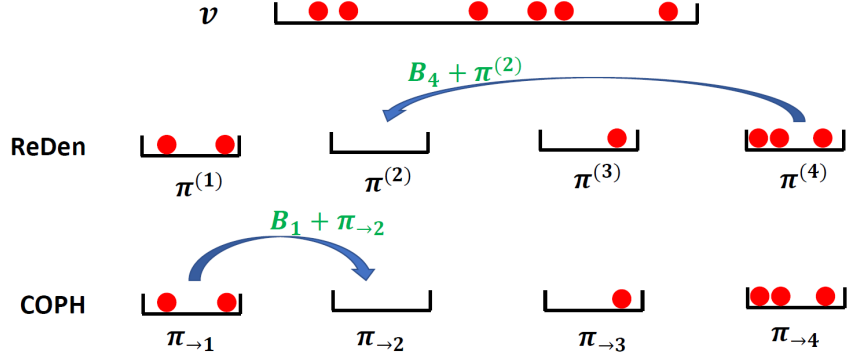}
	\caption{An illustration of the proposed C-OPH strategy. The data vector $\bm v$ is first randomly split into $K=4$ bins. For ReDen~\citep{Proc:Li_NIPS19_BCWS}, we use independent permutation $\pi^{(1)}$, ..., $\pi^{(4)}$ to perform MinHash within each bin. For C-OPH, we use circulant permutations $\pi_{\rightarrow 1}$, ..., $\pi_{\rightarrow 4}$ instead. In this toy example, the 2nd bin $\mathcal B_2$ is empty. In both approaches, we randomly select a non-empty bin ($\mathcal B_4$ for ReDen, and $\mathcal B_1$ for C-OPH), and min-hash the data in that bin with the permutation associated with $\mathcal B_2$.}
	\label{fig:COPH example}
\end{figure}

\subsection{Approach: An Improved Densification Scheme}

As shown in Algorithm~\ref{alg:C-OPH}, the proposed Circulant One Permutation Hashing (C-OPH) method admits the same ``bin-wise hashing + densification'' protocol as ReDen. The key difference is the bin-wise permutations used for hashing. Recall that in Algorithm~\ref{alg:OPH}, to generate each hash $h_k(\bm v)$, we use an independent length-$D/K$ permutation $\pi^{(k)}:[D/K]\mapsto [D/K]$, $k=1,...,M$. In the proposed C-OPH protocol, we indeed only need one length-$D/K$ permutation, $\pi$, which is used circulantly in the same spirit as in C-MinHash. That is, we use $\pi_{\rightarrow k}$ to generated the $k$-th hash, for both the first scan and the densification. We provide an illustrative example in Figure~\ref{fig:COPH example} of C-OPH. Here, to make sure that no circulant permutation is repeatedly used, we assume that $KM\leq D$, which is also needed for establishing rigorous theoretical results.

\subsection{Theoretical Analysis}

We prove that the prior-known minimal Jaccard estimation variance among OPH methods, which is achieved by ReDen, can be further reduced by the proposed C-OPH. For conciseness, we assume $M=K$ and $K^2\leq D$. For both ReDen (Algorithm~\ref{alg:OPH}) and C-OPH (Algorithm~\ref{alg:C-OPH}), the Jaccard estimator between two binary data vectors $\bm v,\bm w\in\{0,1\}^D$ is in the form of average hash collisions,
\begin{align}  \label{eqn:J estimator}
    \hat J(\bm v,\bm w)=\frac{1}{K}\sum_{k=1}^K \mathbbm 1\{h_k(\bm v)=h_k(\bm w)\}.
\end{align}
Denote $\hat J_{ReD}$ and $\hat J_{COPH}$ as the estimator resulting from Algorithm~\ref{alg:OPH} and Algorithm~\ref{alg:C-OPH} (both with same bin selection options as in Algorithm~\ref{alg:C-OPH}), respectively. Suppose the length of each bin $d=D/K$ is an integer. Throughout the paper, we denote
$$a = \sum_{i=1}^D\mathbbm 1\{ \bm v_i = 1 \text{ and }  \bm w_i = 1\}, \quad
f = \sum_{i=1}^D\mathbbm 1\{ \bm v_i = 1 \text{ or }  \bm w_i = 1\}.$$
The following conditional probability regarding the bin split would be needed for our analysis.

\begin{lemma} \label{lemma:a' f'}
Let $d=D/K$. Define function $H$ with the following recursion: $\forall 1\leq k\leq K$, integer $n>0$,
\begin{equation*}
    H(k,n|d)=\sum_{j=1\vee (n-(k-1)d)}^{d \wedge (n-k+1) }{d\choose j}H(k-1,n-j|d),\hspace{0.1in} H(1,n|d)={d\choose n}.
\end{equation*}
Then, conditional on the event that Algorithm~\ref{alg:C-OPH} has $m$ non-empty bins, denote $\tilde a=\sum_{i\in \mathcal B_k}\mathbbm 1\{v_i+w_i=2\}$, $\tilde f=\sum_{i\in \mathcal B_k}\mathbbm 1\{v_i+w_i=1\}$ in a non-empty bin $\mathcal B_k$. For $0\leq a'\leq \min\{a,d\}$, $1\leq f'\leq \min\{f,d\}$, $a'\leq f'$, we have the conditional distribution as
\begin{align}
    P\left[\tilde a=a',\tilde f=f'|m \right]=\frac{{d\choose f'}H(m-1,f-f'|d)}{H(m,f|d)}\cdot\frac{{a\choose a'}{f-a\choose f'-a'}}{{f\choose f'}}. \label{eqn:a' f' dist}
\end{align}
\end{lemma}
\begin{proof}
By Lemma 1 of \cite{Proc:Li_NIPS19_BCWS}, we know that conditional on the event that there are $m$ non-empty bins (denoted as ``$m$'' in the formulas for simplicity), the marginal distribution of $\tilde f$ follows
\begin{align*}
    P\left[\tilde f=f'|m \right]=\frac{{d\choose f'}H(m-1,f-f'|d)}{H(m,f|d)}.
\end{align*}
Thus, the joint distribution can be characterized by
\begin{align*}
    P\left[\tilde a=a',\tilde f=f'|m \right]&=P[\tilde a=a'|\tilde f=f',m]P[\tilde f=f'|m]\\
    &=\frac{{d\choose f'}H(m-1,f-f'|d)}{H(m,f|d)}\cdot\frac{{a\choose a'}{f-a\choose f'-a'}}{{f\choose f'}},
\end{align*}
due to the fact that given $\tilde f$ within a non-empty bin, $\tilde a$ follows a hyper-geometric distribution Hyper($f,a,f',a'$).
\end{proof}

The variance of $\hat J_{COPH}$ is provided as below.

\begin{theorem} \label{theo:variance}
Denote $d=D/K$. Assume $K^2\leq D$ and $f\leq \frac{K-1}{K}D$. Let $N_{emp}=K-m$ be the number of empty bins after the bin split, $J=\frac{a}{f}$ and $\tilde J=\frac{a-1}{f-1}$. We have
\begin{align*}
    Var[\hat J_{COPH}]&=\mathbb E\left[KJ+N_{emp}(2K-N_{emp}-1) E_1 +(K-N_{emp})(K-N_{emp}-1)J\tilde{J} \right],
\end{align*}
with
\begin{align*}
    E_1&=\frac{1}{m}\sum_{(a',f')\in\Omega(m)}\tilde{\mathcal E}(a',f',d)p(a',f'|m)+\frac{m-1}{m}\tilde J J,
\end{align*}
where $\Omega(m)$ is the set of all possible $(\tilde a,\tilde f)$ defined as in Lemma~\ref{lemma:a' f'} within a non-empty bin, and $p(a',f'|m)$ is given by (\ref{eqn:a' f' dist}). Additionally, for any $1\leq a'\leq f'\leq d$,
\begin{align}\notag
    &\tilde{\mathcal E}(a',f',d)=\sum_{\{l_0,l_2,g_0,g_1\}}\Bigg\{\left(\frac{l_0}{f'+g_0+g_1}+\frac{a(g_0+l_2)}{(f'+g_0+g_1)f'}\right) \\\label{eqn:E} &\hspace{0.5in}\times\sum_{s=l}^{d-f'-1} \frac{\binom{d-f'}{s}\binom{f'-a'-1}{d-f'-s-1}\binom{s}{n_1}\binom{d-f'-s}{n_2}\binom{d-f'-s}{n_3}\binom{f'-a'-(d-f-s)}{n_4}\binom{a'-1}{a'-l_1-l_2}}{\binom{d-a'-1}{d-f'-1}\binom{d-1}{a'}}\Bigg\}, \notag
\end{align}
where
\begin{align*}
    &n_1=g_0-(d-f'-s-g_1), \hspace{0.33in} n_2=d-f'-s-g_1,\\
    &n_3=l_2-g_0+(d-f'-s-g_1), \hspace{0.08in} n_4=l_1-(d-f'-s-g_1),
\end{align*}
and the feasible set of $\{l_0,l_2,g_0,g_1\}$ satisfies the intrinsic constraints that for non-negative $l|_0^2,g|_0^2,h|_0^2$,
\begin{align*}
    &l_0+l_1+l_2=l_0+g_0+h_0=a',\\
    &g_0+g_1+g_2=l_1+g_1+h_1=d-f',\\
    &h_0+h_1+h_2=l_2+g_2+h_2=f'-a'.
\end{align*}
In particular, it holds that $Var[\hat J_{COPH}]<Var[\hat J_{ReD}]$.
\end{theorem}
\begin{remark}
In Theorem~\ref{theo:variance}, the expectation is with respect to the number of empty bins, $N_{emp}$. The exact probability distribution of this random variable is
$$Pr \left[ N_{emp} = j\right]= \sum_{\ell=0}^{K-j}(-1)^\ell {K\choose j}{K-j\choose \ell}
{D(1-(j+\ell)/K)\choose f}\bigg/{D\choose f}.$$
We refer interested readers to \cite{Proc:Li_NIPS19_BCWS} for the detailed derivation.
\end{remark}

\begin{proof}
Firstly, we separate the event of matching hash values into two distinct events. Denote $C_k^E$ the indicator of hash collision at bin $k$ when $k$ is empty, and $C_k^N$ the indicator of collision when $k$ is simultaneously non-empty. Consequently, we have $C_k=C_k^N+C_k^E$. Denote $I_{emp,k}$ as the indicator function of the $k$-th bin being empty. According to \cite{Proc:Li_Owen_Zhang_NIPS12}, for non-empty bins we have
$$\mathbb E(C_k^N|I_{emp,k}=0)=\mathbb E(C_k^E|I_{emp,k}=0)=J,$$
$$\mathbb E[(C_k^N)^2|I_{emp,k}=0]=\mathbb E[(C_k^E)^2|I_{emp,k}=0]=J.$$
Based on above notations, for both schemes we can write $\hat J=\frac{1}{K}\sum_{k=1}^K(C_k^E+C_k^N)$. For any unbiased estimator, we can expand
\begin{equation} \label{var}
Var(\hat J)=\mathbb E[\frac{1}{K^2}(\sum_{k=1}^K(C_k^E+C_k^N))^2]-J^2\triangleq \frac{1}{K^2}A-J^2.
\end{equation}
It suffices to analyze $A$. Conditional on the event that the number of non-empty bins $K-N_{emp}=m$,~we~have
$$A=\mathbb E[\mathbb E[(\sum_{k=1}^K(C_k^E+C_k^N))^2|K-N_{emp}=m]].$$
For the ease of notation, we simply let $m$ denote the event $\{K-N_{emp}=m\}$. One important fact of C-OPH is that, the expectation of $C_i^NC_j^N$ for non-empty bins $\mathcal B_i, \mathcal B_j$ where $\pi$ is used circulantly equals that of ReDen where two independent permutations are applied to $\mathcal B_i$ and $\mathcal B_j$ respectively. This is because in the algorithms, the bin split is uniform at random. Keeping this in mind, for both ReDen and C-OPH we have
\begin{align}
A&=\mathbb E[\mathbb E[(\sum_{k=1}^M(C_k^E+C_k^N))^2|m]] \nonumber\\
&=\mathbb E[\mathbb E[\sum_{k=1}^K[(C_k^E)^2+(C_k^N)^2]+\sum_{i\neq j}C_i^EC_j^E  \nonumber\\
&\hspace{1.3in} +2\sum_{i\neq j}C_i^EC_j^N+\sum_{i\neq j}C_i^NC_j^N|m]] \nonumber\\
&=\mathbb E[KJ+N_{emp}(N_{emp}-1)E_1+2N_{emp}(K-N_{emp})E_1   \nonumber\\
&\hspace{1.3in} +(K-N_{emp})(K-N_{emp}-1)J\tilde{J}]
 \nonumber\\
&=\mathbb E[KJ+N_{emp}(2K-N_{emp}-1) E_1 \nonumber \\
&\hspace{1.3in} +(K-N_{emp})(K-N_{emp}-1)J\tilde{J}], \label{var-2}
\end{align}
where $E_1=\mathbb E[C_i^EC_j^E|m]$, for $i\neq j$. Since $f\leq \frac{K-1}{K}D$, this expectation is always positive. Denote event $\Upsilon$ as bin $i$ and bin $j$ choosing the same non-empty bin in the densification bin selection procedure. Expanding the expectation yields
\begin{align*}
    E_1&=P[C_i^E=C_j^E=1|m]\\
    &=P[C_i^E=C_j^E=1|\Upsilon,m]P[\Upsilon|m]+P[C_i^E=C_j^E=1|\Upsilon^c,m]P[\Upsilon^c|m]\\
    &=\frac{1}{m}P[C_i^E=C_j^E=1|\Upsilon,m]+\frac{m-1}{m}\tilde J J.
\end{align*}
Suppose both bins select bin $k$ (non-empty). Let $\mathcal I_k$ be the index set of bin $k$. Denote $a_k=|\{i\in\mathcal I_k:v_i+w_i=2\}|$, $f_k=|\{i\in\mathcal I_k:v_i+w_i=1\}|$ as the number of two types of simultaneously non-zero elements in bin $k$. Let $\Omega(m)$ be the collection of possible $(a_k,f_k)$ conditional on $m$ non-empty bins. We proceed with
\begin{align}
    E_1&=\frac{1}{m}\sum_{(a',f')\in\Omega(m)}\Big\{P\left[C_i^E=C_j^E=1|(a_k,f_k)=(a',f'),\Upsilon,m \right]  \notag\\
    &\hspace{0.7in} \times P\left[(a_k,f_k)=(a',f')|\Upsilon,m \right]\Big\}+\frac{m-1}{m}\tilde J J.  \label{eqn:E1}
\end{align}
In (\ref{eqn:E1}), the second probability is precisely the conditional probability given by Lemma~\ref{lemma:a' f'}. The first probability can be computed as $\tilde{\mathcal E}(a',f',d)$ by Lemma 3.3 in~\cite{CMH2Perm2021}, applied to the selected non-empty bin for densification. We refer interested readers to \cite{CMH2Perm2021} for details. Aggregating parts together gives the exact formula of the variance of C-OPH Jaccard estimator.

To show that $Var[\hat J_{COPH}]$ is smaller than $Var[\hat J_{ReD}]$, notice that in (\ref{eqn:E1}), the distributions of $(a_k,f_k)$ and $m$ only depend on the bin split procedure which is shared by both ReDen and C-OPH. Therefore, using circulant permutations in fact only affects the term
\begin{align*}
    B\triangleq P[C_i^E=C_j^E=1|(a_k,f_k)=(a',f'),\Upsilon,m],
\end{align*}
for each $(a',f')$, consider bin $k$ as a length-$D/K$ data vector with within-bin Jaccard index $J'=a'/f'$. Since $K^2\leq D$, applying the arguments of Theorem 3.4 in~\cite{CMH2Perm2021} within the bin, we know that C-OPH gives smaller term $B$ than ReDen by the variance reduction of circulant permutations applied within the bin. Now combining (\ref{var}) and (\ref{var-2}) proves the desired result.
\end{proof}

Theorem~\ref{theo:variance} says that, C-OPH can further improve the prior known smallest variance of the OPH-type Jaccard estimators, which is the main theoretical merit of the proposed C-OPH. It is easy to show that this strict variance reduction holds for arbitrary $K$ and $M$ with $MK\leq D$, which is usually true in practice for high-dimensional data. Moreover, even if this condition is not present, the variance of C-OPH can still be considerably smaller. In line 8 and line 16 of Algorithm~\ref{alg:C-OPH}, we can change ``$\pi_{\rightarrow k}$'' into ``$\pi_{\rightarrow (k+\lfloor kd/D \rfloor)}$'', where the term $\lfloor kd/D \rfloor$ is to deploy an additional periodic shifting. For example, suppose $d=\frac{D}{K}=\frac{M}{2}$ (i.e., $MK=2D$). In the first $k=d=\frac{M}{2}$-th hashes, the permutation used to generate hash from, e.g., $\mathcal B_1$, is $\pi_{\rightarrow 1}$. With the shifting term, the permutation used in the last $\frac{M}{2}$ hash values for $\mathcal B_1$ would be $\pi_{k+\rightarrow \lfloor k/K \rfloor}=\pi_{\rightarrow (d+2)}$, which is different from the first permutation $\pi_{\rightarrow 1}$ used for $\mathcal B_1$. Otherwise (without this shifting term), since $\pi_{\rightarrow (d+1)}=\pi_{\rightarrow 1}$, the permutations (for $\mathcal B_1$) will repeat, leading to exactly the same hash values and consequently larger estimation variance. However, if $\mathcal B_1$ is an empty bin, then we still have $1/m$ chance of using the same permutation to generate hashes from a same non-empty bin (densification bin collision), where $m$ is the number of non-empty bins in total. This would slightly increase the variance of C-OPH from the rigorous theoretical perspective. Nonetheless, in most practical cases where $f=|\bm v\cup \bm w|$ is not very small (i.e., not extremely sparse data in the absolute scale), the impact of the densification bin collision would be negligible (e.g., see our empirical verification). To sum up, the variance reduction of C-OPH is valid rigorously when $MK\leq D$, and is true empirically  even when $MK>D$.

\begin{figure}[h]
	\begin{center}
		\mbox{\hspace{-0.1in}
		\includegraphics[width=2.25in]{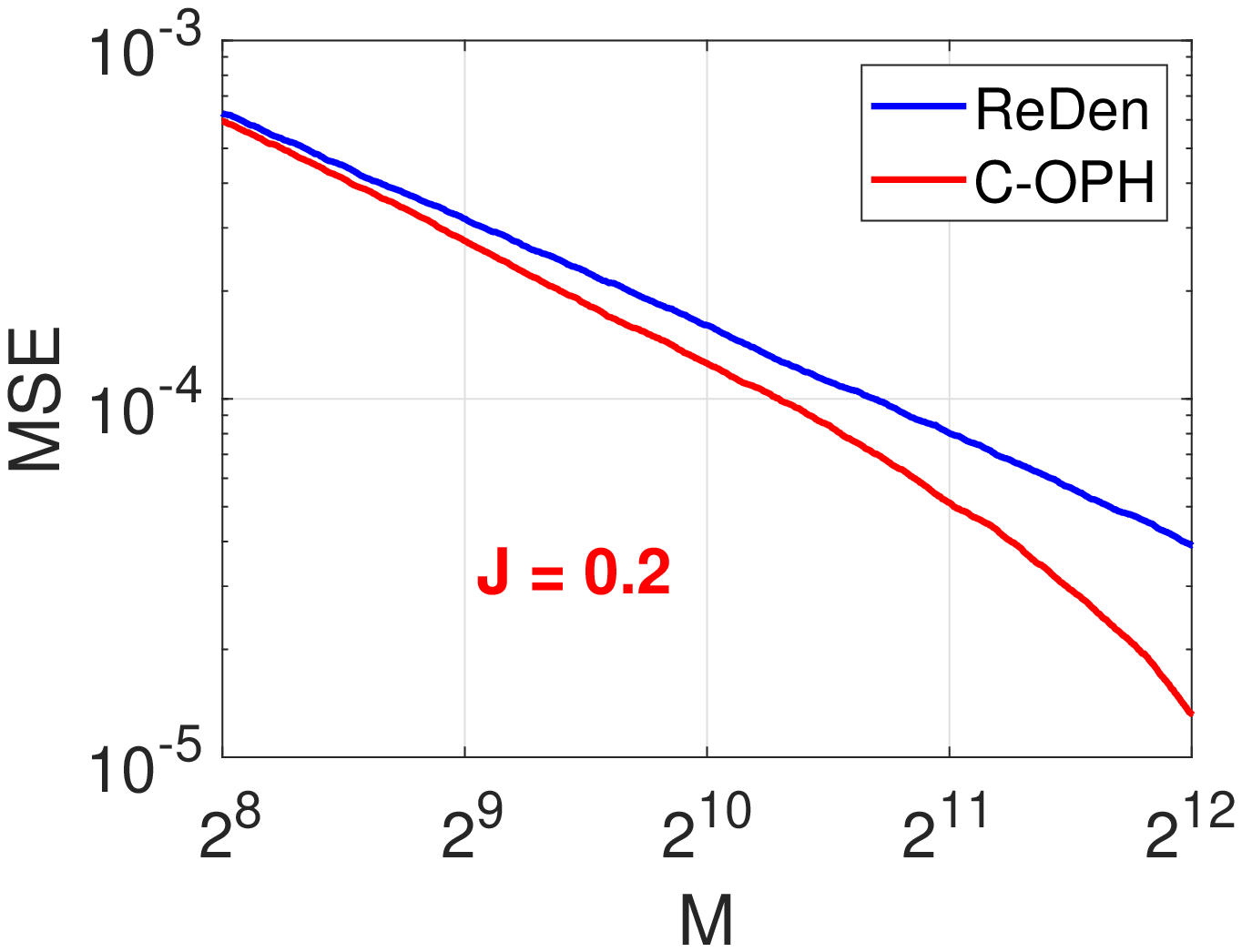}\hspace{-0.1in}
		\includegraphics[width=2.25in]{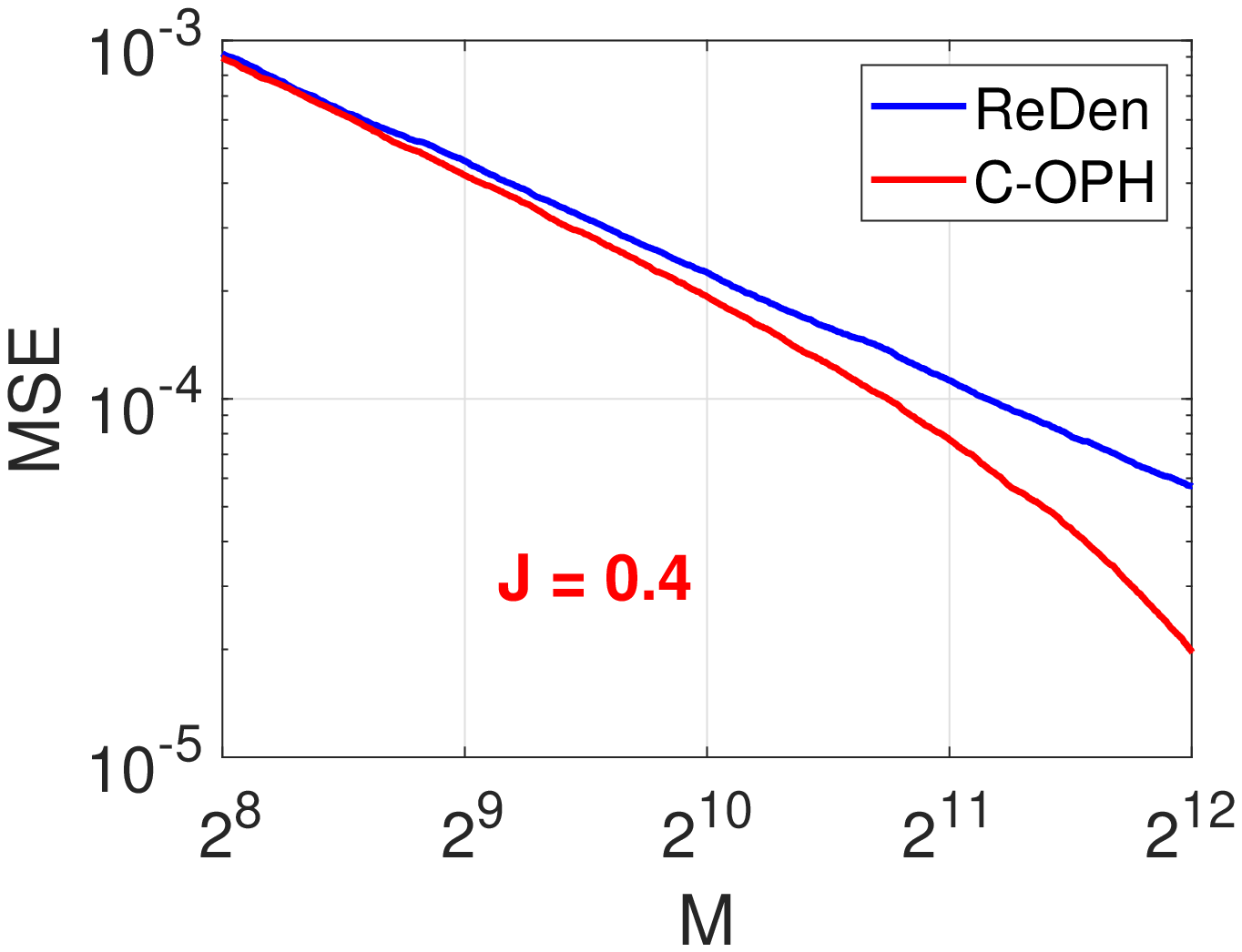}\hspace{-0.1in}
		\includegraphics[width=2.25in]{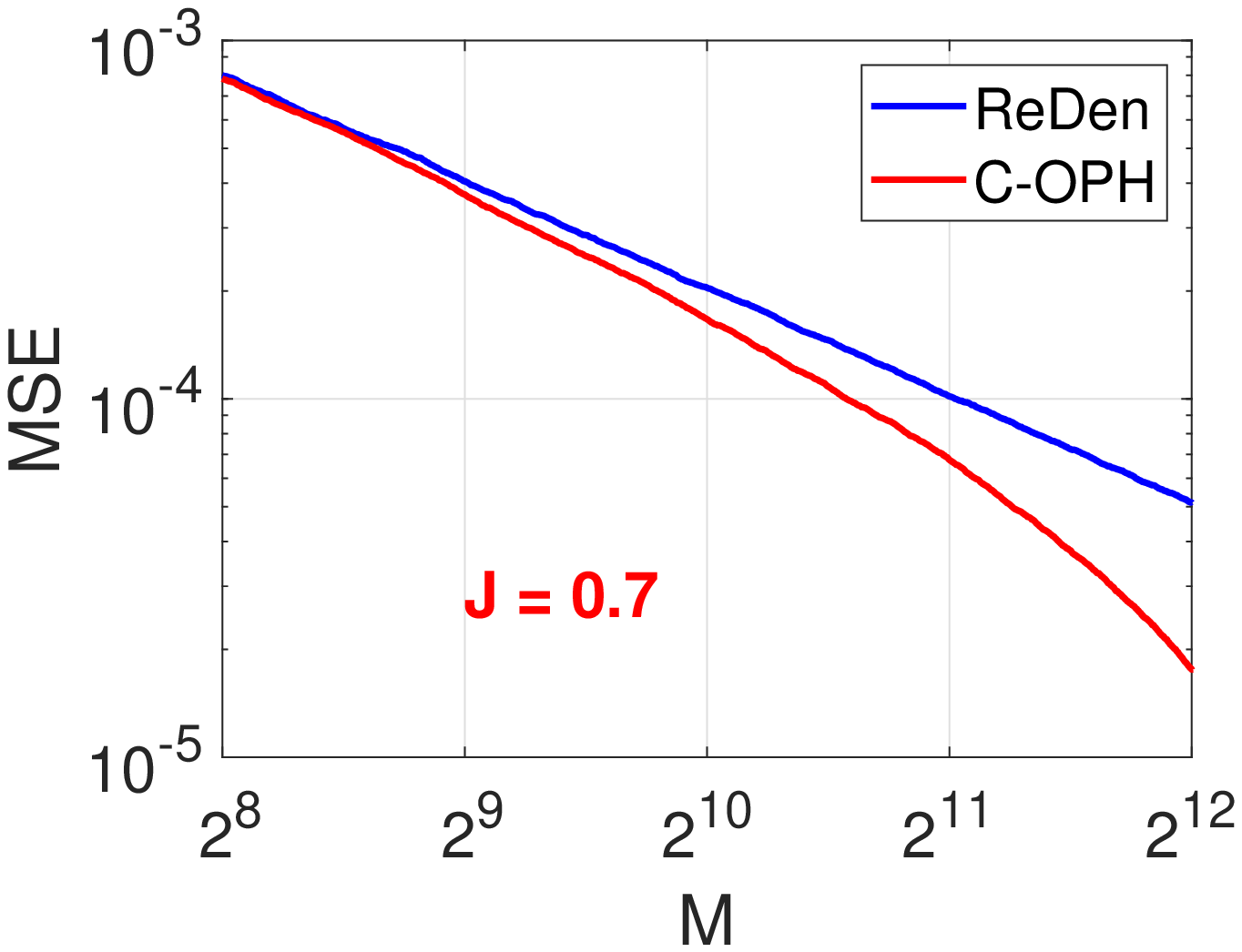}
		}
		\mbox{\hspace{-0.1in}
		\includegraphics[width=2.25in]{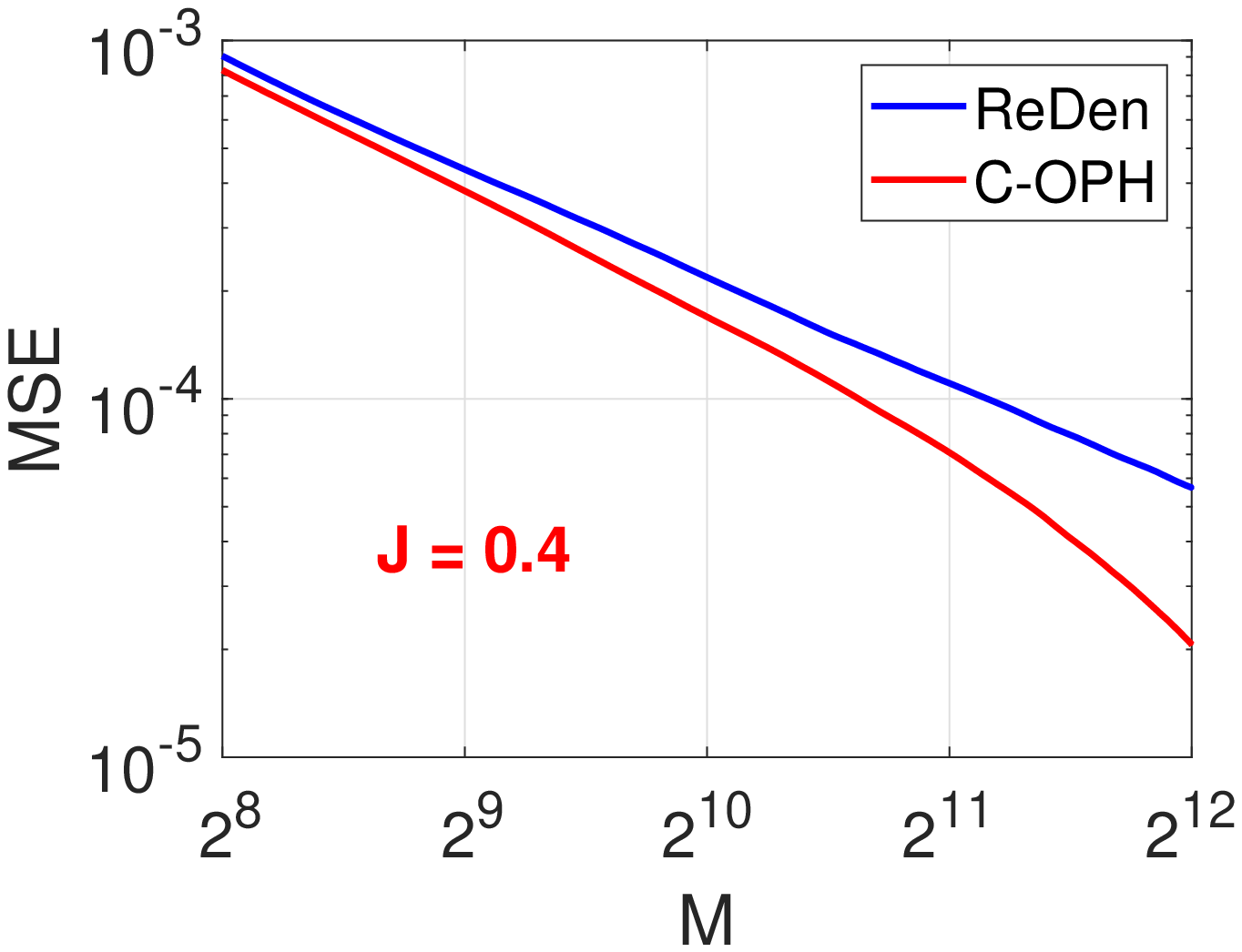}\hspace{-0.1in}
		\includegraphics[width=2.25in]{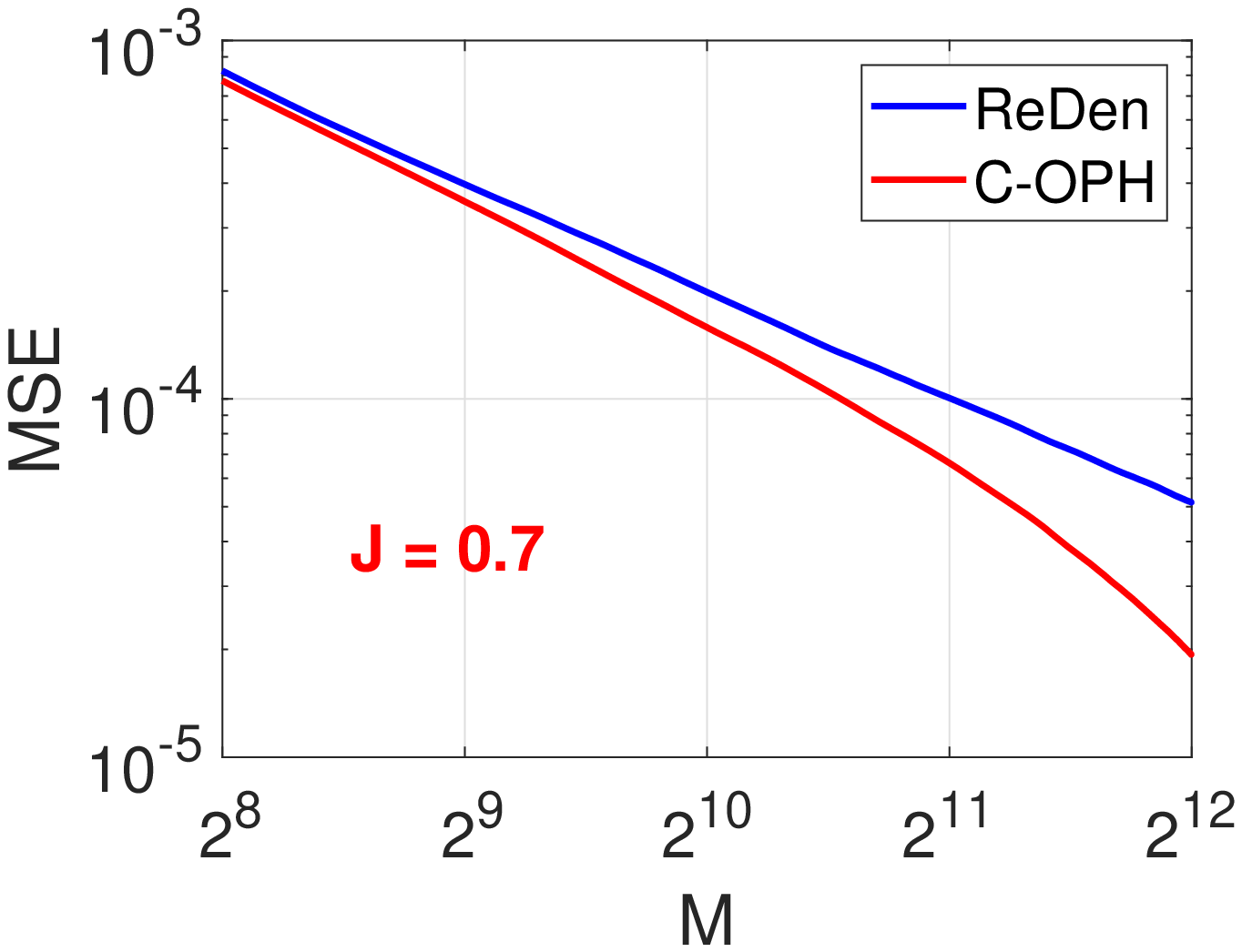}\hspace{-0.1in}
		\includegraphics[width=2.25in]{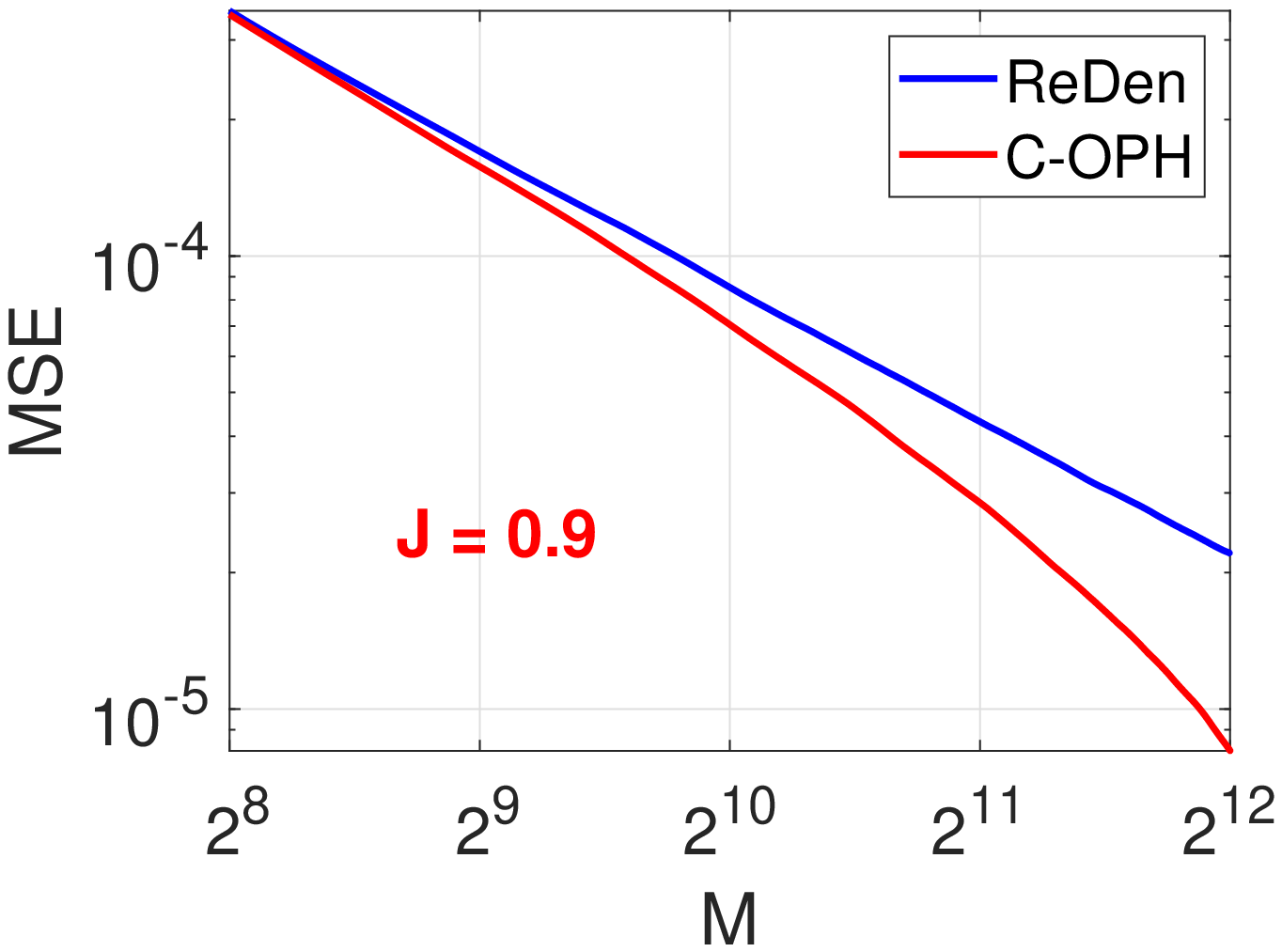}
		}
	\end{center}
	
	\vspace{-0.25in}
	
	\caption{Empirical Mean Squared Error (MSE) of simulated data pairs, $D=2^{12}$, $f=|\bm v\cup \bm w|=2^{11}$, with various $J$. \textbf{1st row:} $K=2^5$. \textbf{2nd row:} $K=2^7$. We confirm that the MSE of C-OPH is smaller than that of ReDen in all cases.}
	\label{fig:MSE COPH sim}	\vspace{-0.5in}
\end{figure}

In Figure~\ref{fig:MSE COPH sim}, we plot the empirical MSE of C-OPH versus ReDen on several synthetic binary data pairs to demonstrate the variance reduction effect and justify the theory. Here we simulate data vectors with various dimensionality $D$ and Jaccard similarity values. The bin split is realized by a random permutation on all data elements. We validate that C-OPH can indeed incur smaller estimation variance than ReDen. More numerical examples on real-world data can be found in Section~\ref{sec:practice} (Figure~\ref{fig:MSE COPH words K128} and~\ref{fig:MSE COPH words K512}).

\section{Practical Implementation}  \label{sec:practice}

So far, our discussion has assumed that we have access to a uniformly random bin split. One can certainly achieve this by doing a random permutation on the data vector, and then split non-zero elements into different bins based on their permuted location. In practice, however, applying exact permutations might not be desirable when the data dimensionality is extremely large, due to the high memory consumption to store the ``permutation vector''. Instead, using hash functions to approximate random permutations is a common choice, which waves the need to store explicitly the permutation. In this section, we provide some discussion on the practical implementation of C-OPH, to show how hash functions can be effectively used in C-OPH. Recall that, C-MinHash~\citep{CMH2Perm2021,CMH1Perm2021} can reduce the $K$ permutations used in standard MinHash to merely one permutation, allowing the efficient usage of one ``permutation vector'' such that the practical estimation always matches the theory. Interestingly, when combined with OPH, our proposed C-OPH can further reduce the number of permutations used to just ``$1/K$'', practically speaking.

\subsection{When shall we use hash functions?}

In many applications involving very high ($D$) dimensional data, storing exact random permutations may be too costly. A simple alternative is to use simple hash functions to approximate the permutations. It is well-known that, in some cases, simply replacing the permutations with hash functions in MinHash may lead to undesirable and systematic estimation inconsistencies. However, interestingly as shown in Figure~\ref{fig:MSE CMH words}, for C-MinHash-$(\sigma,\pi)$, we can safely replace the first permutation $\sigma$ by simple 2-universal (2U) hashing.

\begin{figure}[h]

\begin{center}

		\mbox{\hspace{-0.1in}
		\includegraphics[width=2.25in]{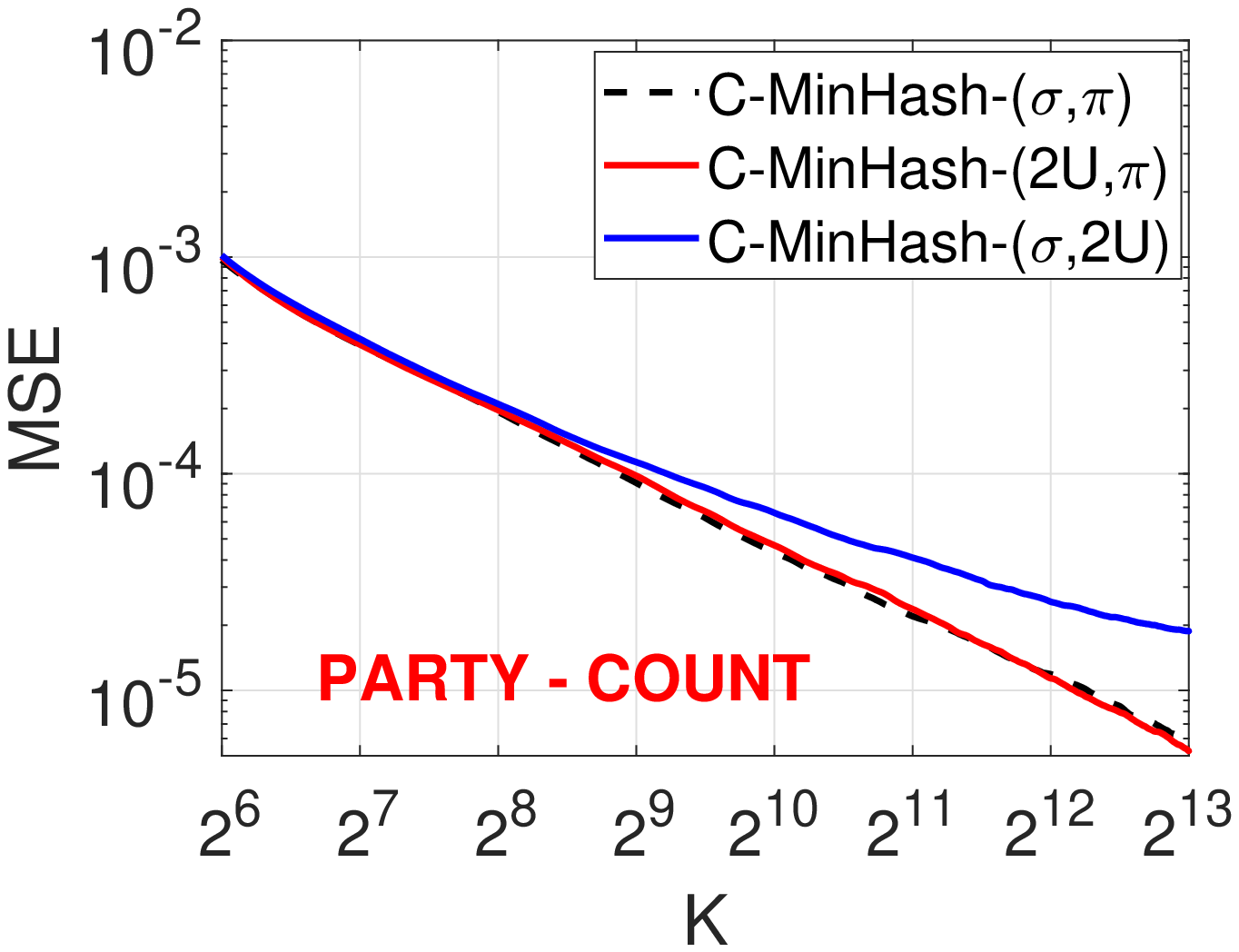}\hspace{-0.1in}
		\includegraphics[width=2.25in]{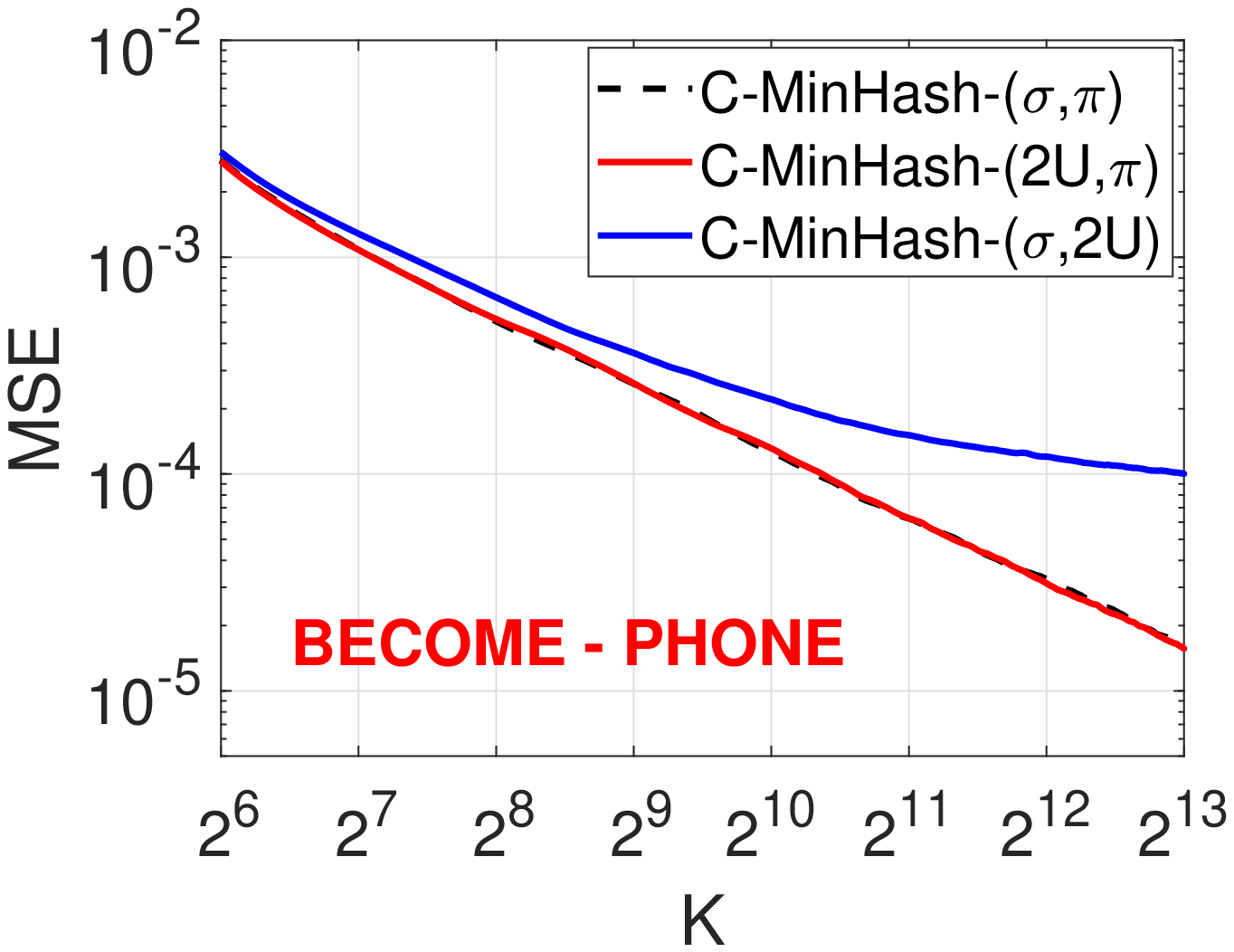}\hspace{-0.1in}
		\includegraphics[width=2.25in]{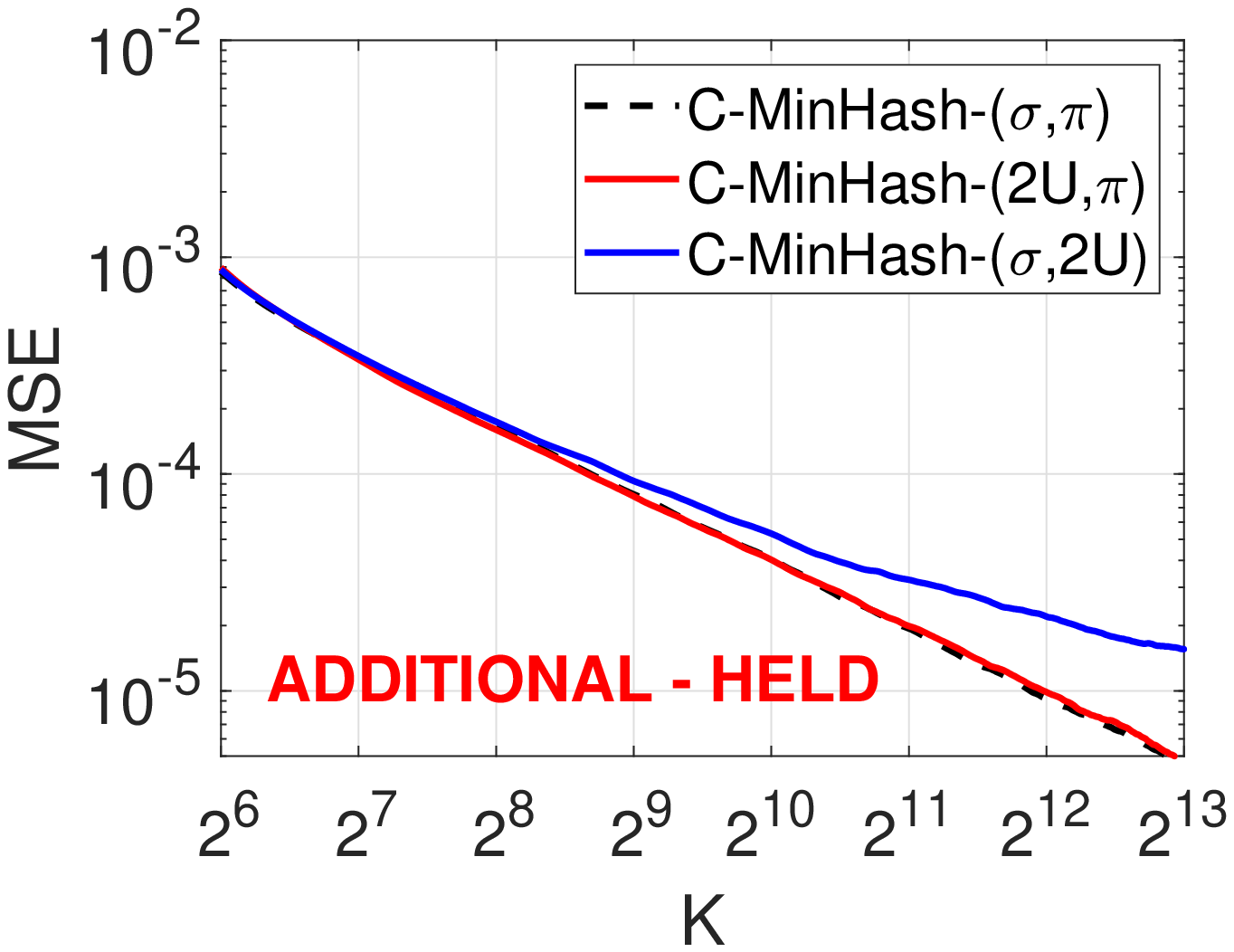}
		}
	\end{center}
	
	\vspace{-0.2in}
	
	\caption{Mean Squared Error (MSE) of C-MinHash-$(\sigma,\pi)$, C-MinHash-$(2U,\pi)$ and C-MinHash-$(\sigma,2U)$ on data pairs from the \textit{Words} dataset~\citep{Proc:Li_Church_EMNLP05}, with various $J$. We see that replacing the initial (pre-processing) permutation $\sigma$ by a 2U hash function (i.e., C-MinHash-$(2U,\pi)$)  gives virtually the same estimation MSE as C-MinHash-$(\sigma,\pi)$. However, replacing the second permutation $\pi$ by 2U (i.e., C-MinHash-$(\sigma,2U)$) leads to larger estimation errors. In each plot, the two vectors from~\citet{Proc:Li_Church_EMNLP05} denote whether the words appear in the documents of the repository. }
	\label{fig:MSE CMH words}
\end{figure}

Here, the 2U hash function is defined by $\mathcal H:[D]\mapsto [D]$ for a non-negative integer $x\in [D]$:     $\mathcal H(x)=ax+b\ \text{mod}\ p$, where $p>D$ is a prime number,  $a$ and $b$ are uniformly chosen from $\{0,1,...,p-1\}$ and $a$ is odd. There are also known tricks to avoid the modulo operations.  As shown in Figure~\ref{fig:MSE CMH words}, for C-MinHash-$(\sigma,\pi)$, we can safely replace by the first permutation $\sigma$ by 2U hash. However, we can observe  obvious performance deviations  once we place the second permutation $\pi$ by 2U hash (i.e., C-MinHash-$(\sigma,2U)$). We have experimented with other simple hashing functions including 4-universal hashing and murmur hashing, but the observations are essentially very similar to using 2U hashing.

\subsection{From One Permutation to ``$1/K$'' Permutation}

Therefore, we naturally develop  C-OPH-$(2U,\pi)$. That is, we can replace the initial permutation $\sigma$ in the original C-OPH-$(\sigma,\pi)$. This means, we actually just need ``$1/K$'' permutation instead of one permutation. This might need to a substantial convenience in pratice. For example, if $D=2^{40}$ and $K=2^{10}$, then we just need one short permutation of size $D/K = 2^{30}$, which is small enough even for GPU memory.

Figure~\ref{fig:MSE COPH words K128} and Figure~\ref{fig:MSE COPH words K512} provide another set of experiments for sanity check: (i) C-OPH-$(\sigma,\pi)$ improves ``ReDen'' the prior best densification~\citep{Proc:Li_NIPS19_BCWS}; (2) The ``$1/K$ permutation version C-OPH-$(2U,\pi)$ performs essentially the same as C-OPH-$(\sigma,\pi)$, using the same ``Words'' dataset~\citep{Proc:Li_Church_EMNLP05}. From the plots with various $K$, we validate again that the proposed C-OPH provides smaller estimation MSEs than ReDen. Furthermore, using hash function for bin split (C-OPH-$(2U,\pi)$) essentially gives same empirical accuracy as the standard C-OPH-$(\sigma,\pi)$.

\begin{figure}[h]
	\begin{center}
		\mbox{\hspace{-0.1in}
		\includegraphics[width=2.25in]{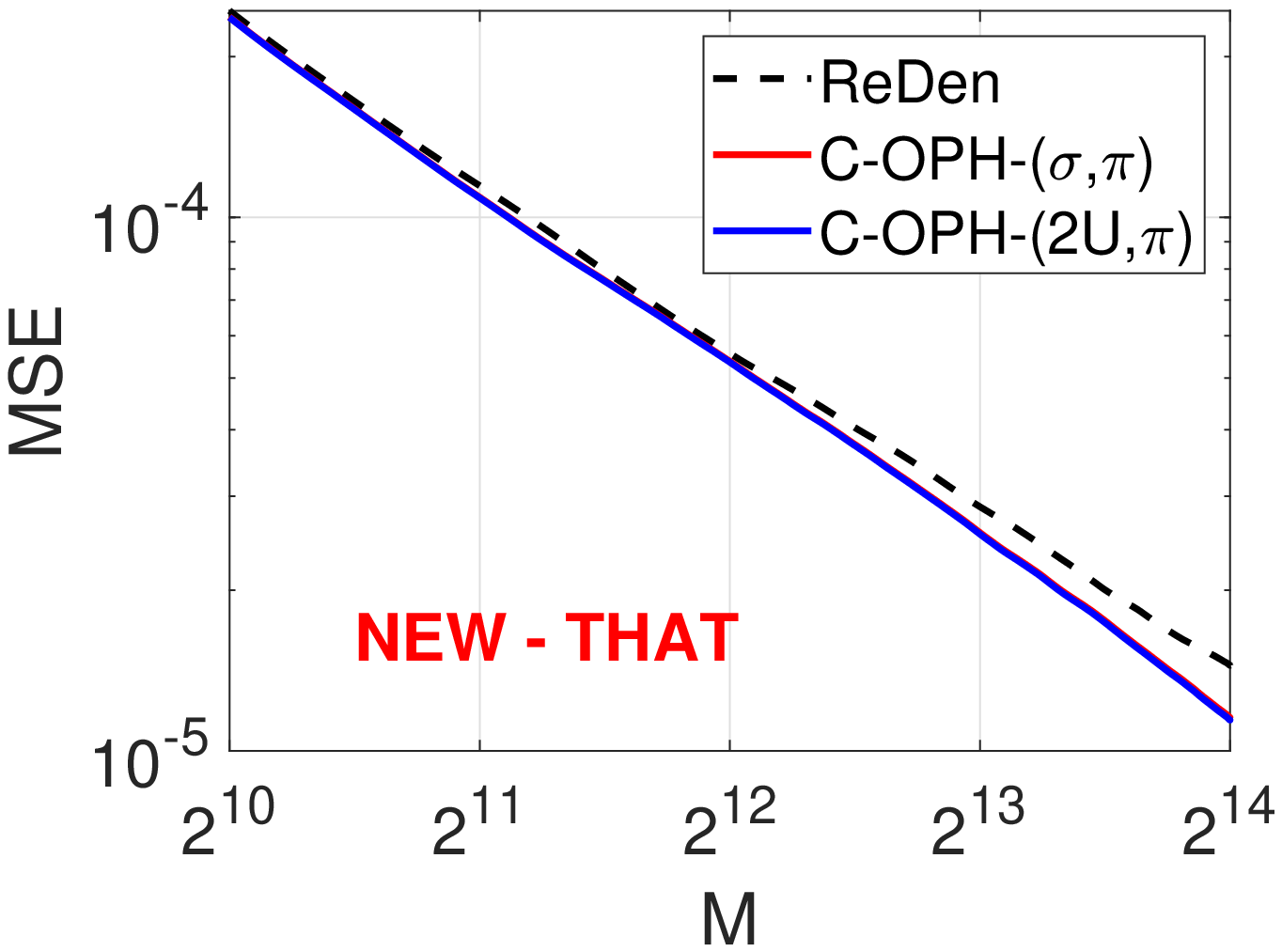}\hspace{0.1in}
		\includegraphics[width=2.25in]{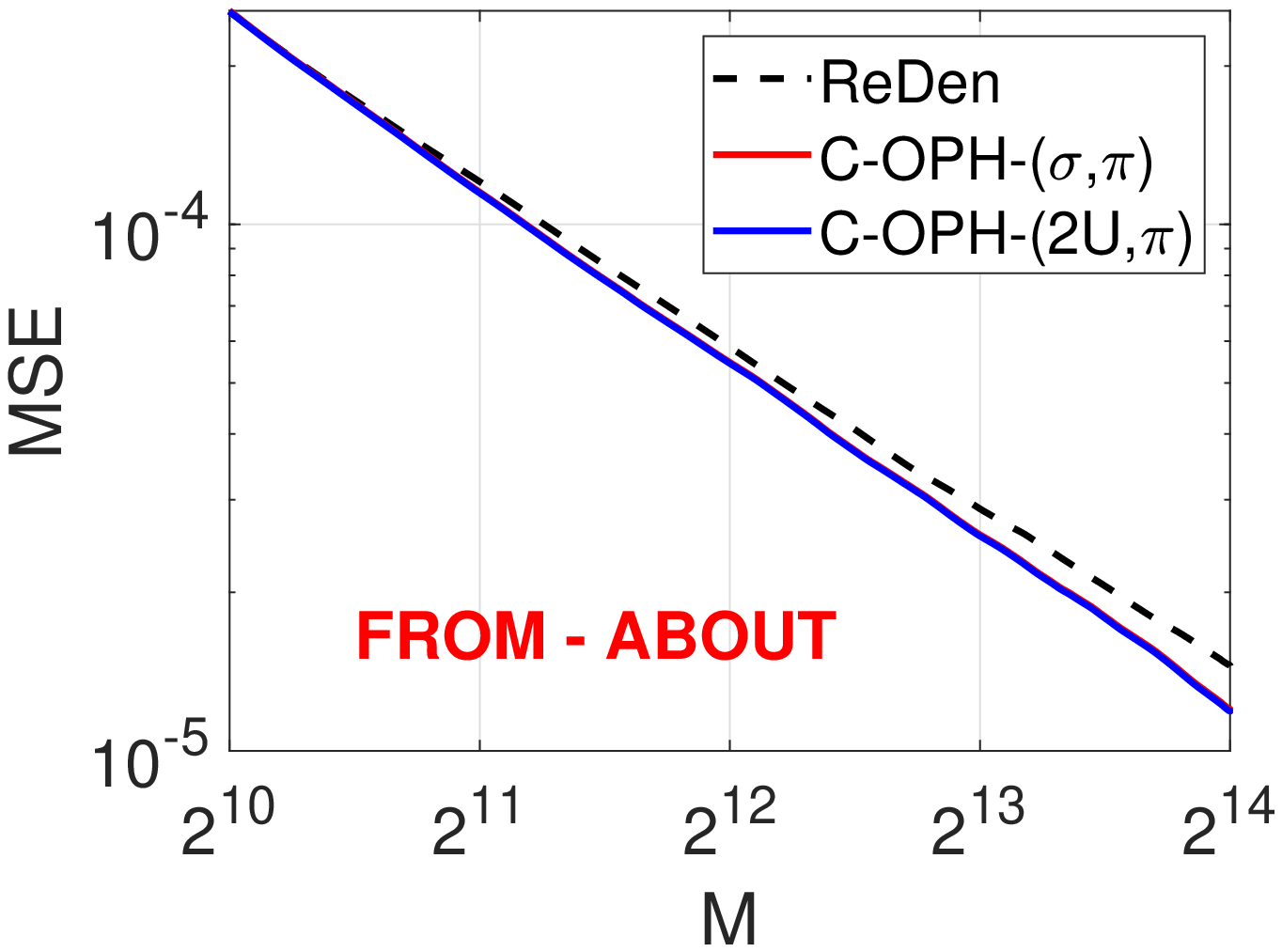}\hspace{0.1in}
		\includegraphics[width=2.25in]{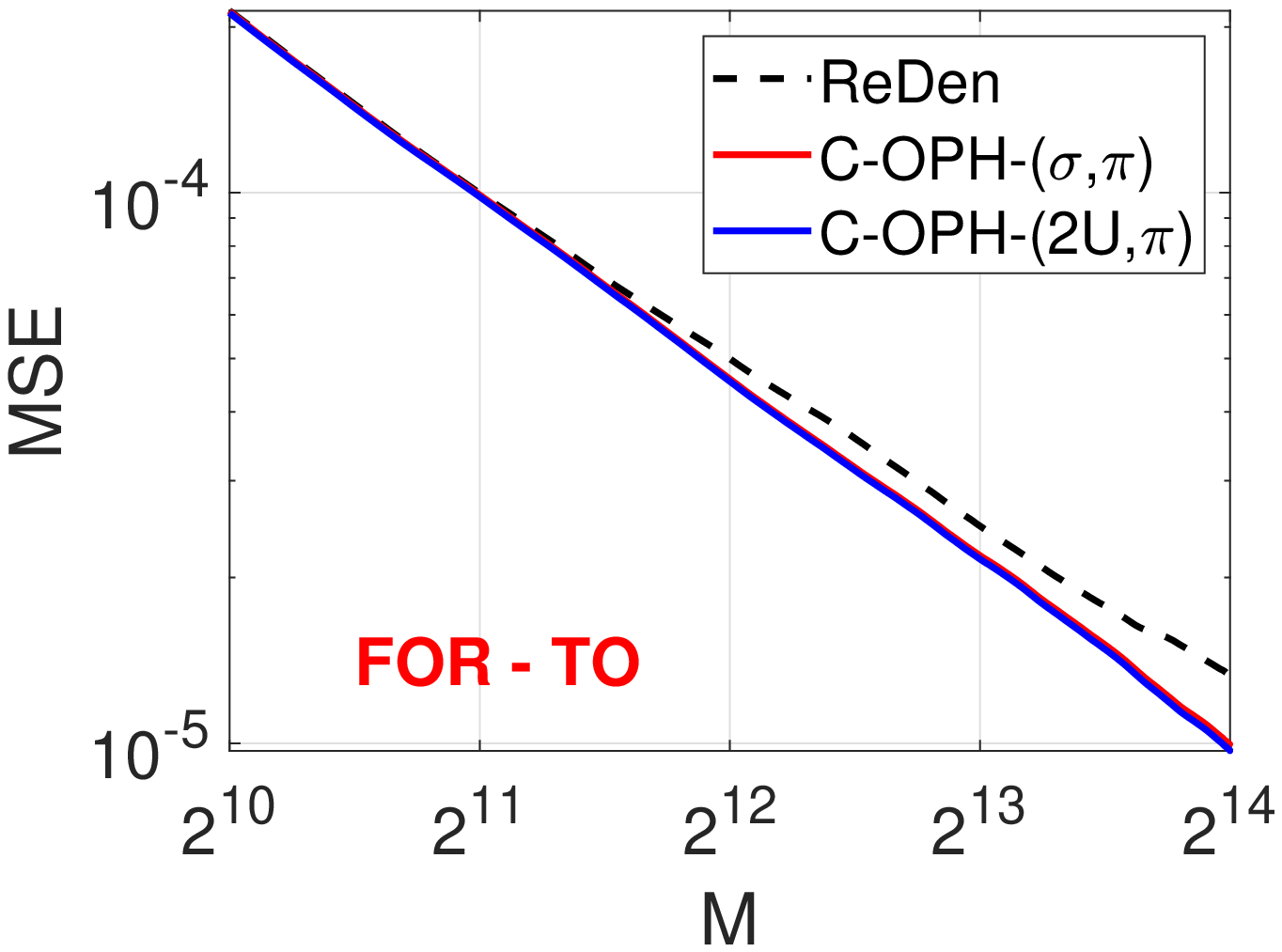}
		}
		\mbox{\hspace{-0.1in}
		\includegraphics[width=2.25in]{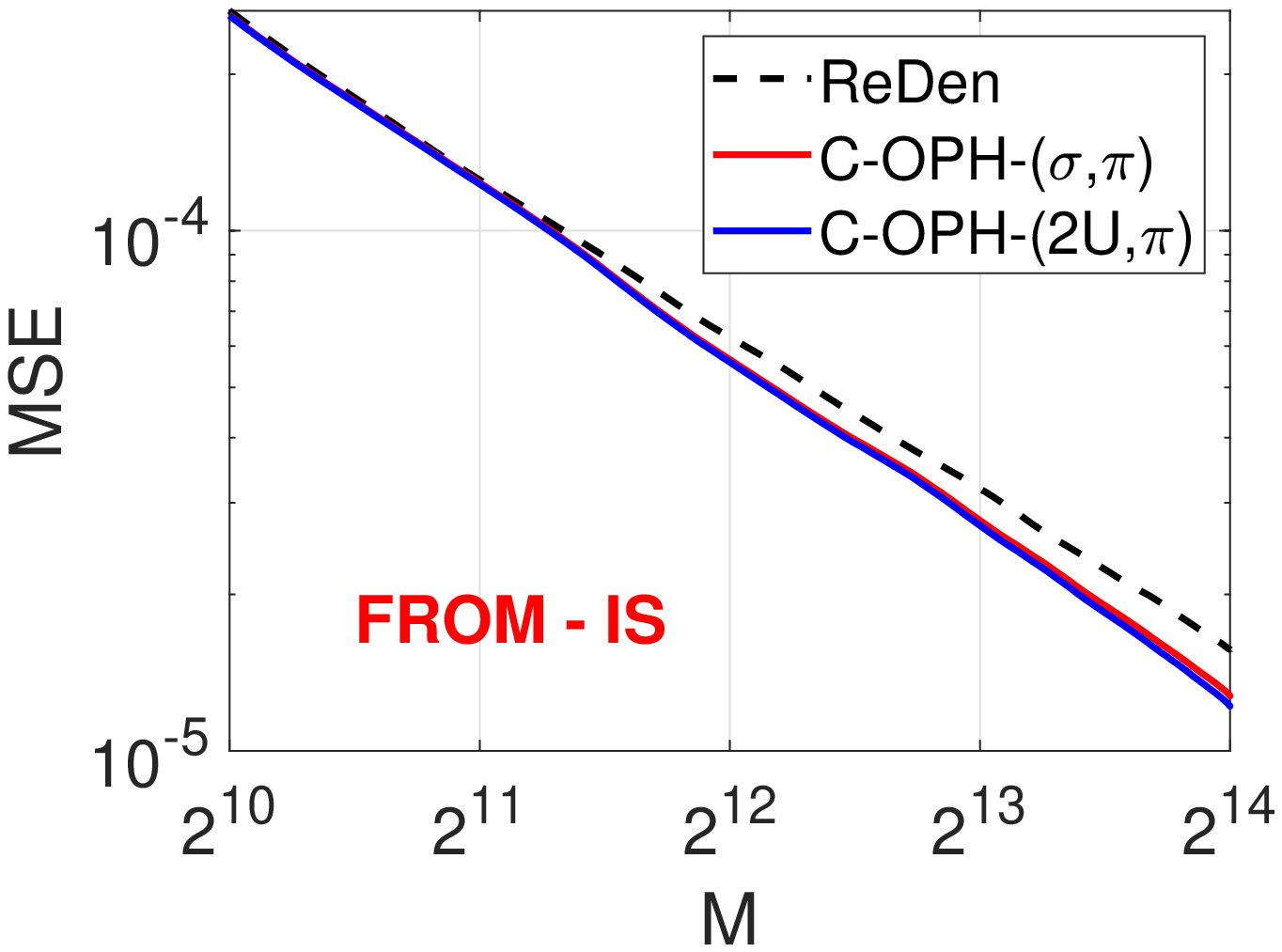}\hspace{0.1in}
		\includegraphics[width=2.25in]{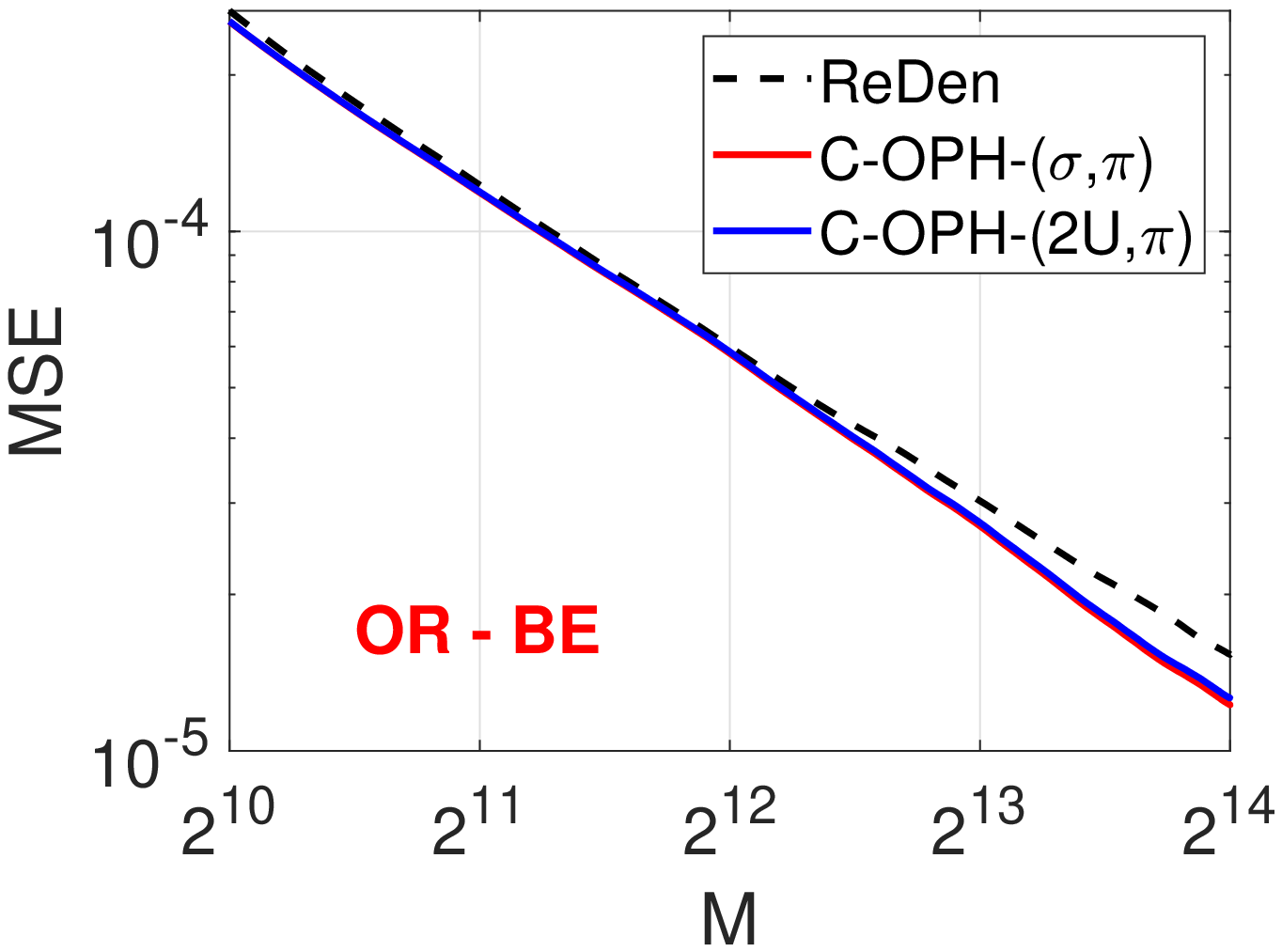}\hspace{0.1in}
		\includegraphics[width=2.25in]{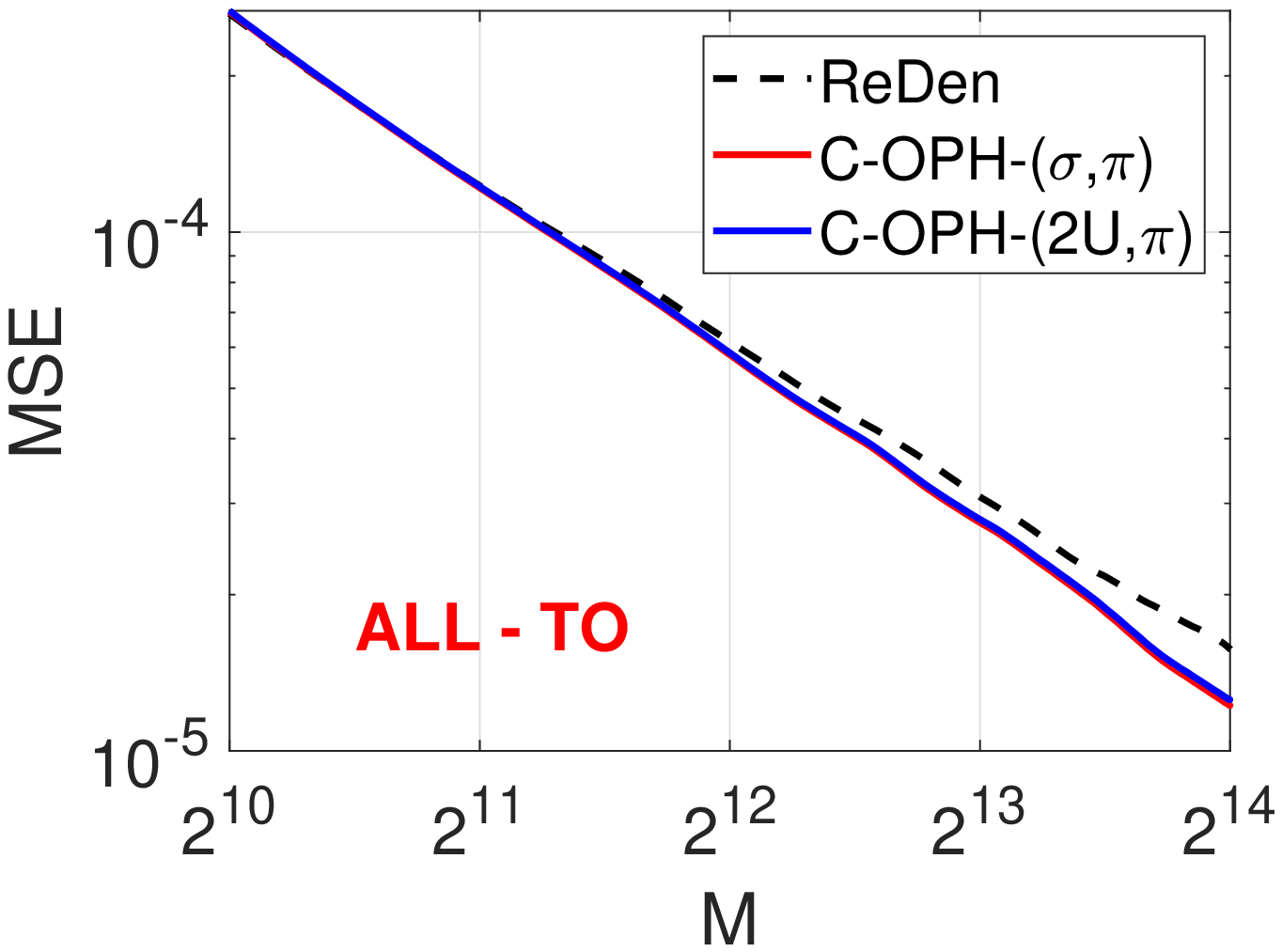}
		}
		\mbox{\hspace{-0.1in}
		\includegraphics[width=2.25in]{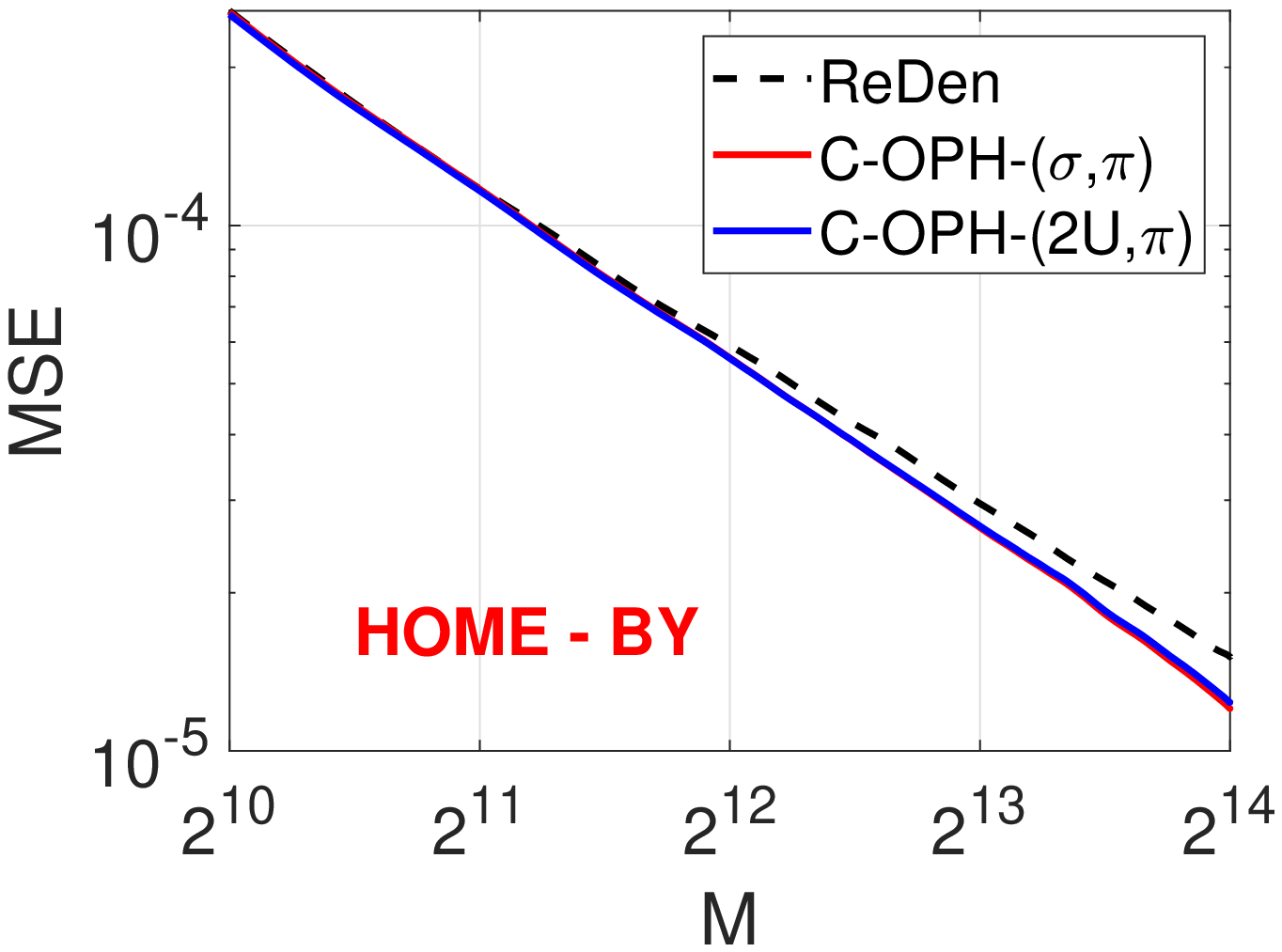}\hspace{0.1in}
		\includegraphics[width=2.25in]{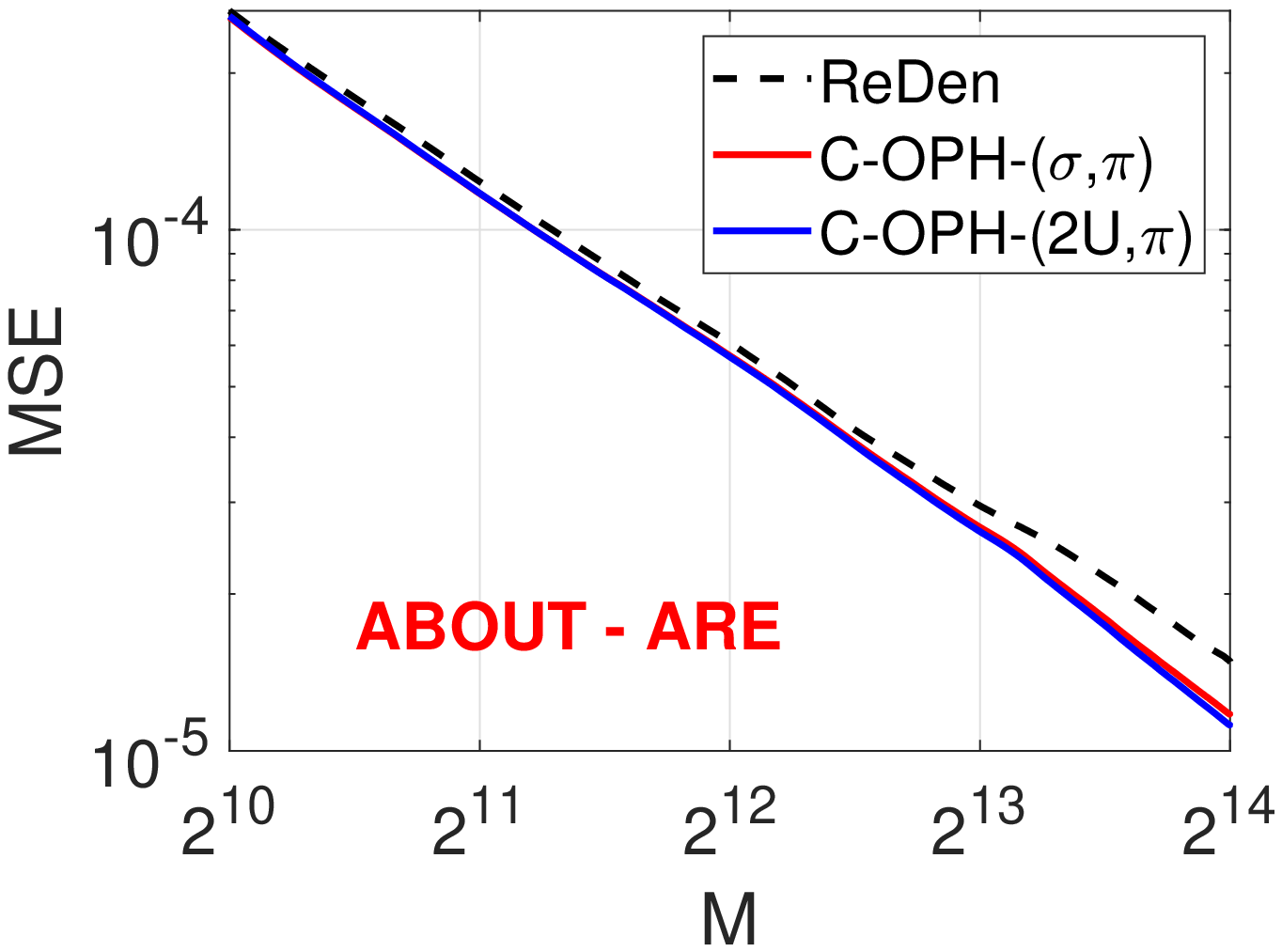}\hspace{0.1in}
		\includegraphics[width=2.25in]{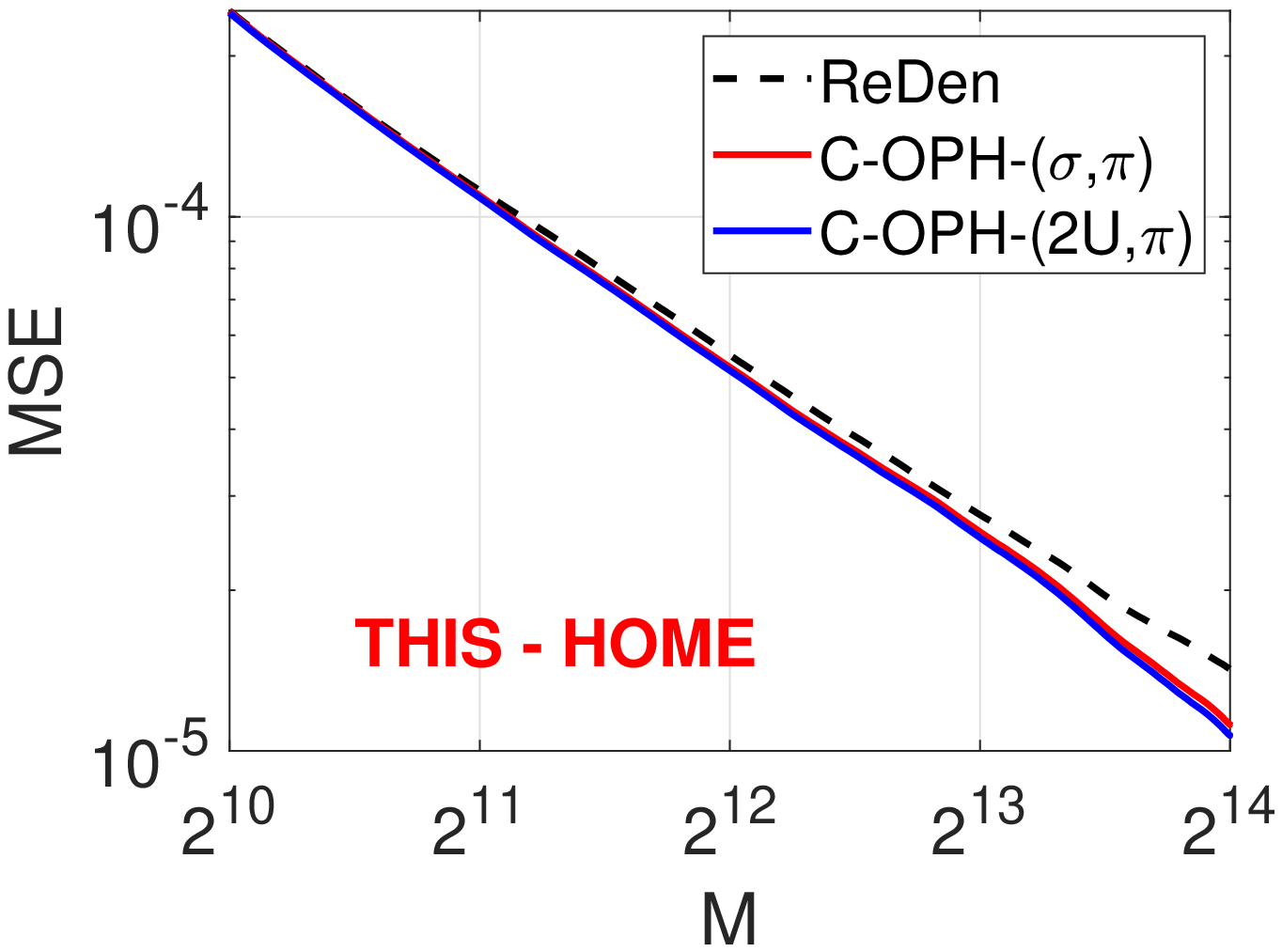}
		}
	\end{center}
	\vspace{-0.1in}
	\caption{Mean Squared Error (MSE) of ReDen, C-OPH-$(\sigma,\pi)$ and C-OPH-$(2U,\pi)$ on word pairs from the \textit{Words} dataset, $K=2^7$. We see that C-OPH improves the MSE of ReDen, and using hash function to perform bin split empirically gives same MSE as using perfectly random permutation.}
	\label{fig:MSE COPH words K128}\
\end{figure}

\newpage\clearpage

\begin{figure}[t!]
	\begin{center}
		\mbox{\hspace{-0.1in}
		\includegraphics[width=2.25in]{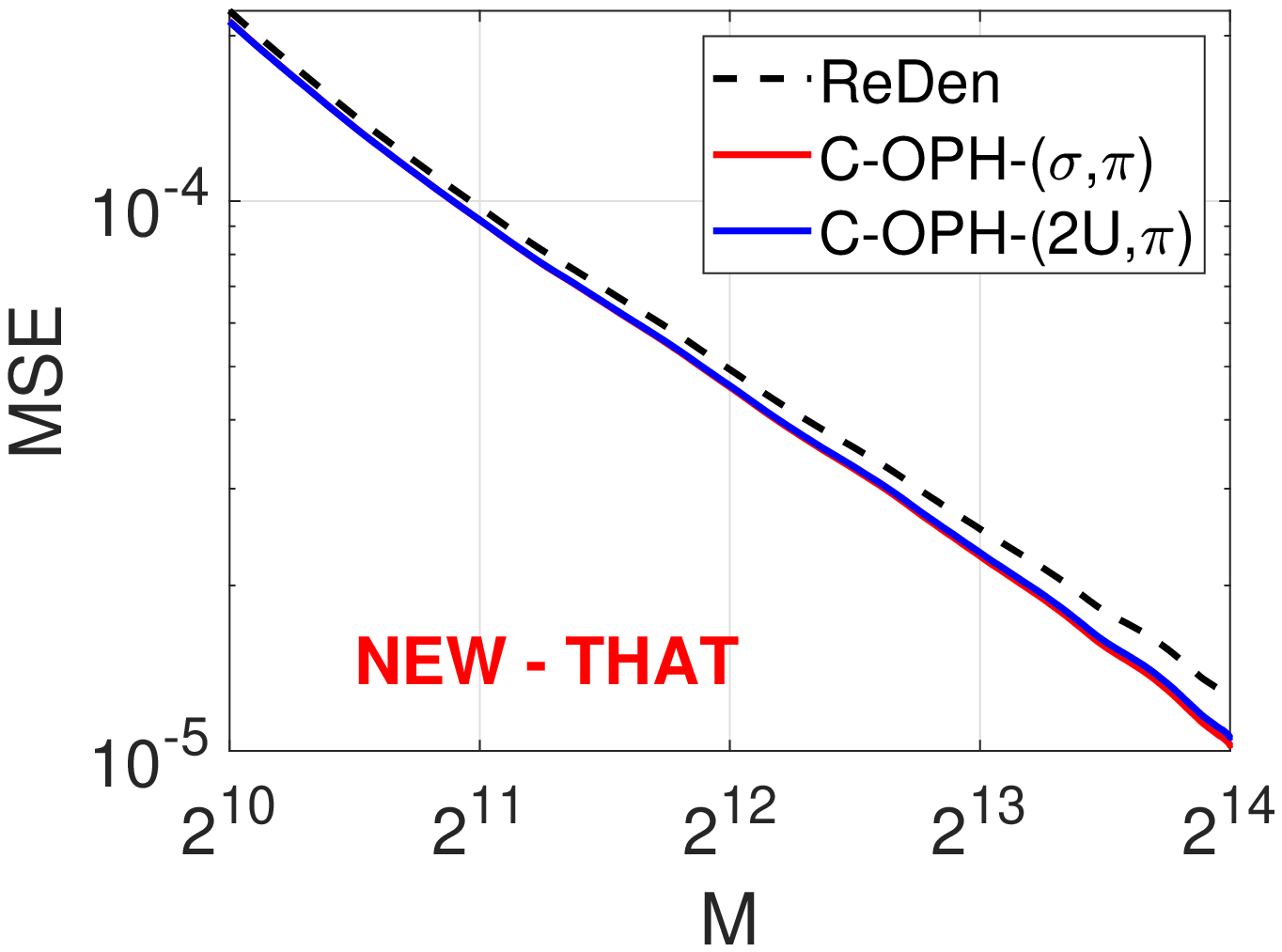}\hspace{-0.1in}
		\includegraphics[width=2.25in]{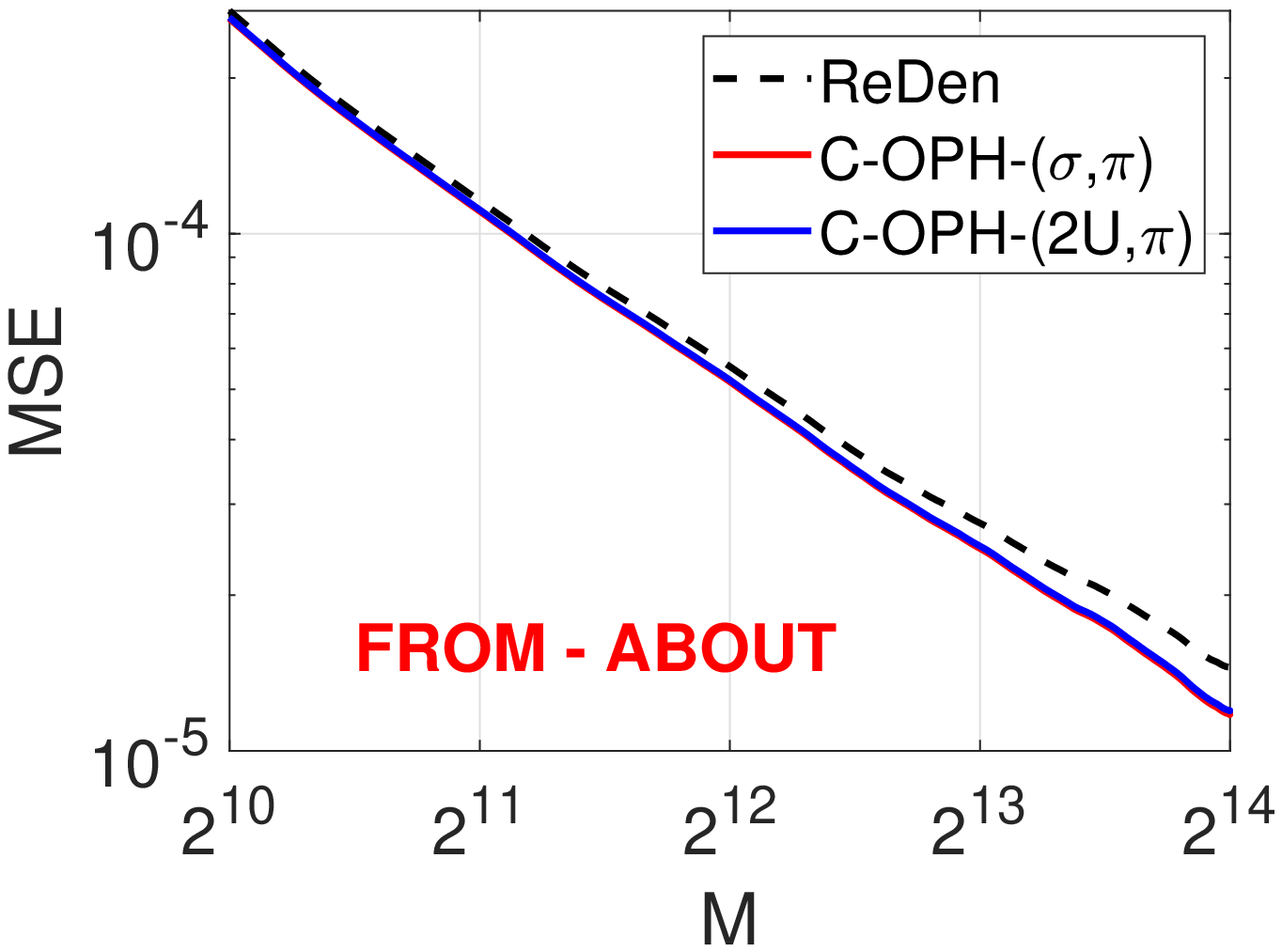}\hspace{-0.1in}
		\includegraphics[width=2.25in]{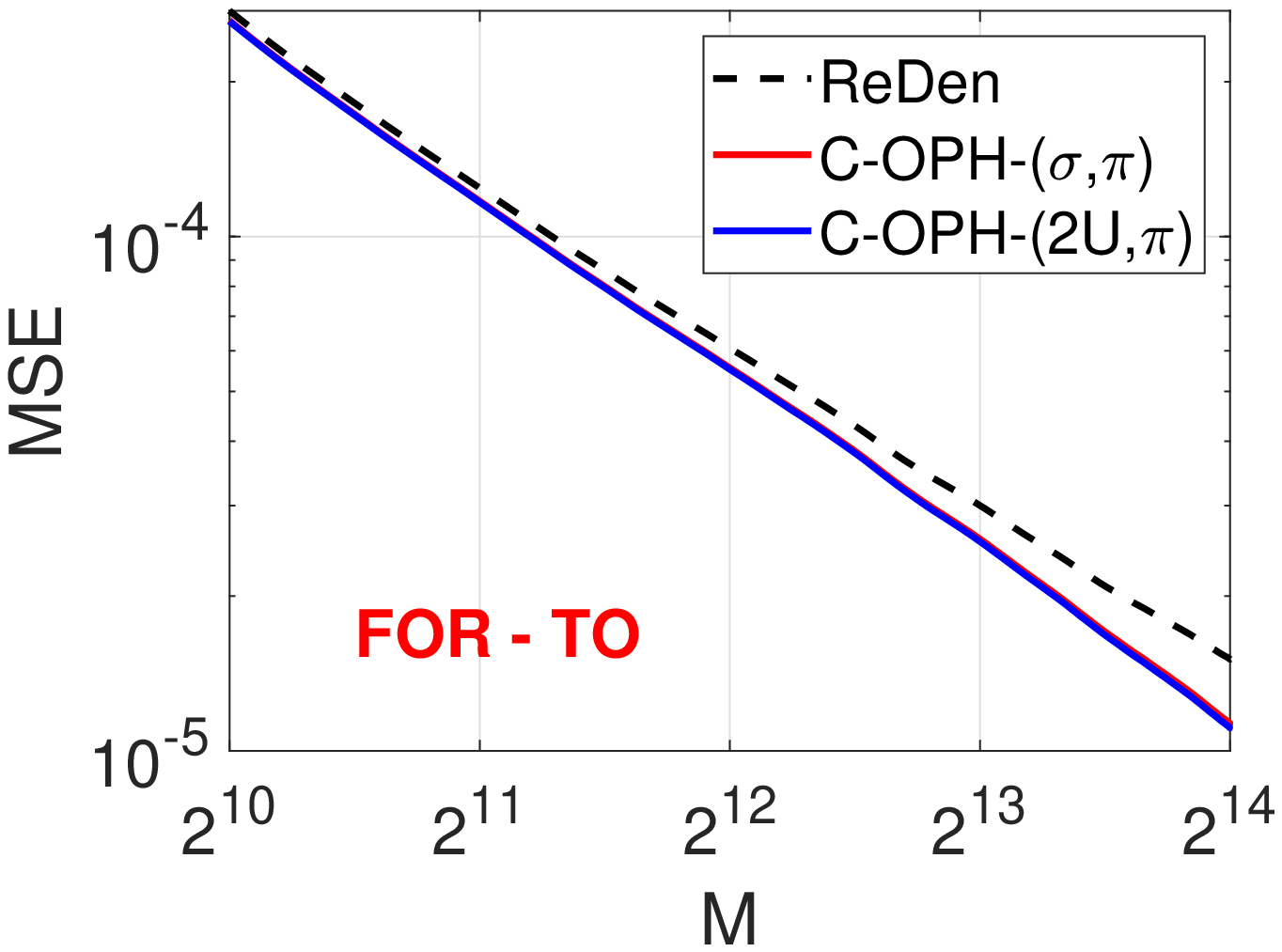}
		}
		\mbox{\hspace{-0.1in}
		\includegraphics[width=2.25in]{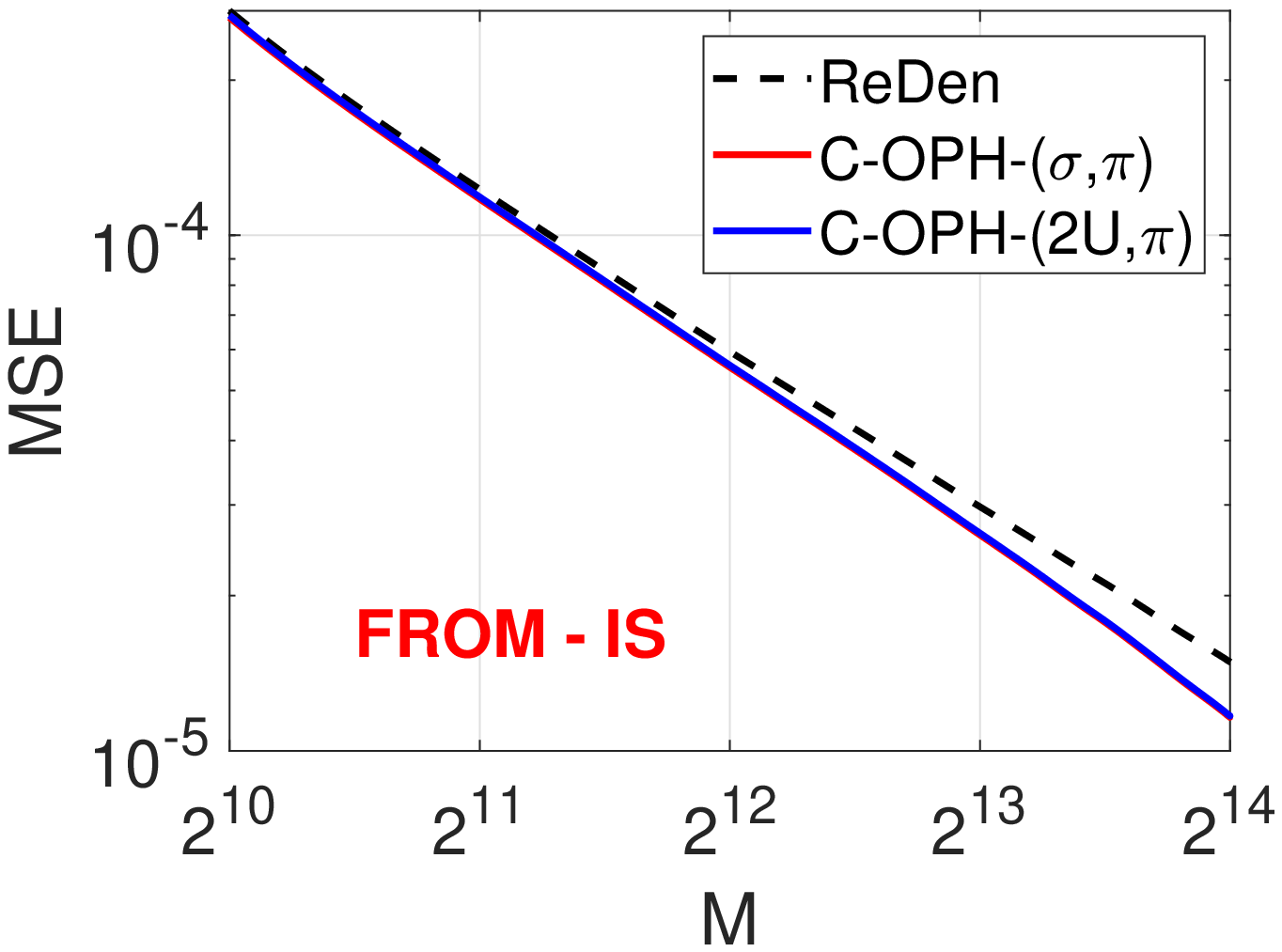}\hspace{-0.1in}
		\includegraphics[width=2.25in]{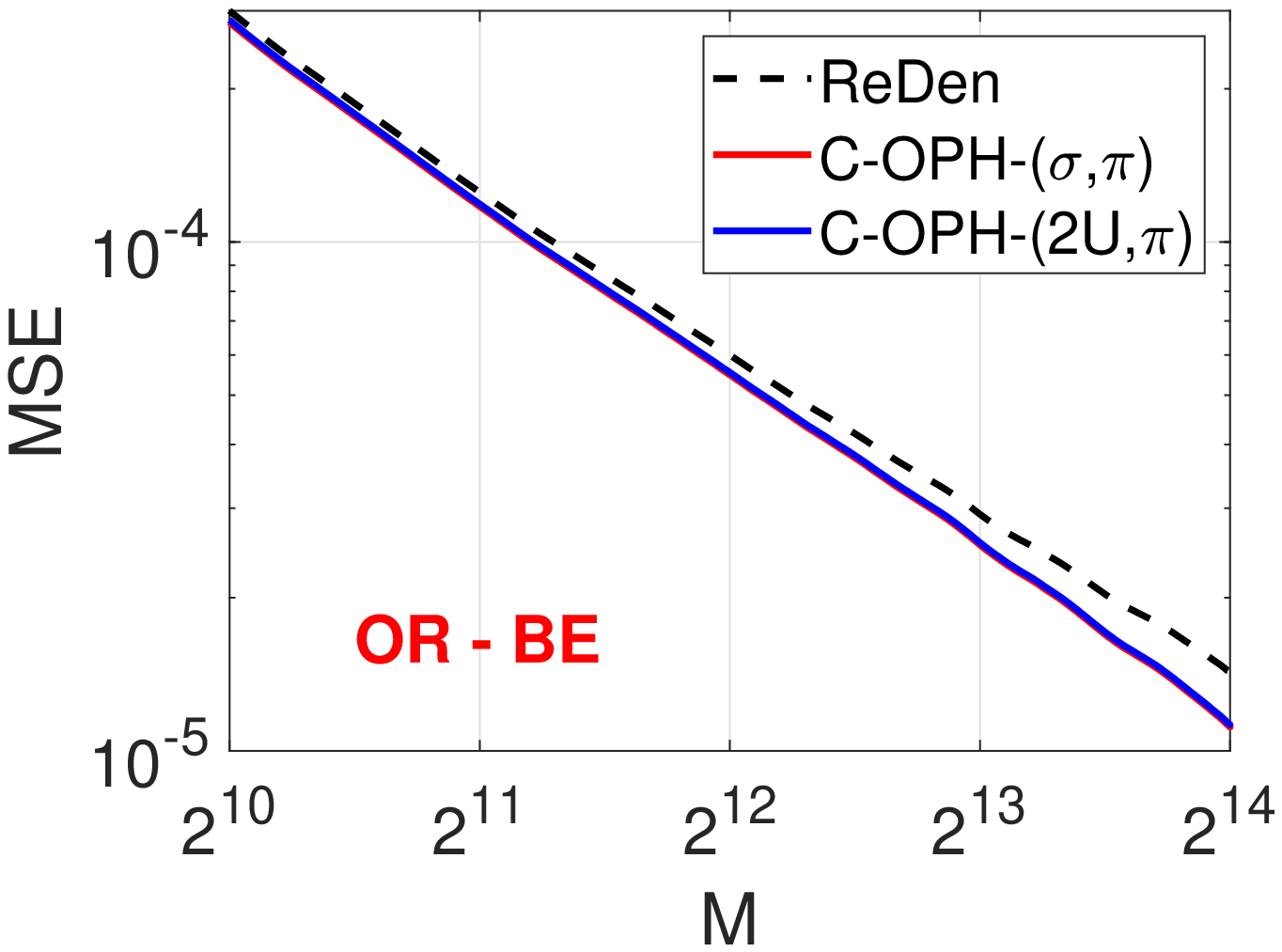}\hspace{-0.1in}
		\includegraphics[width=2.25in]{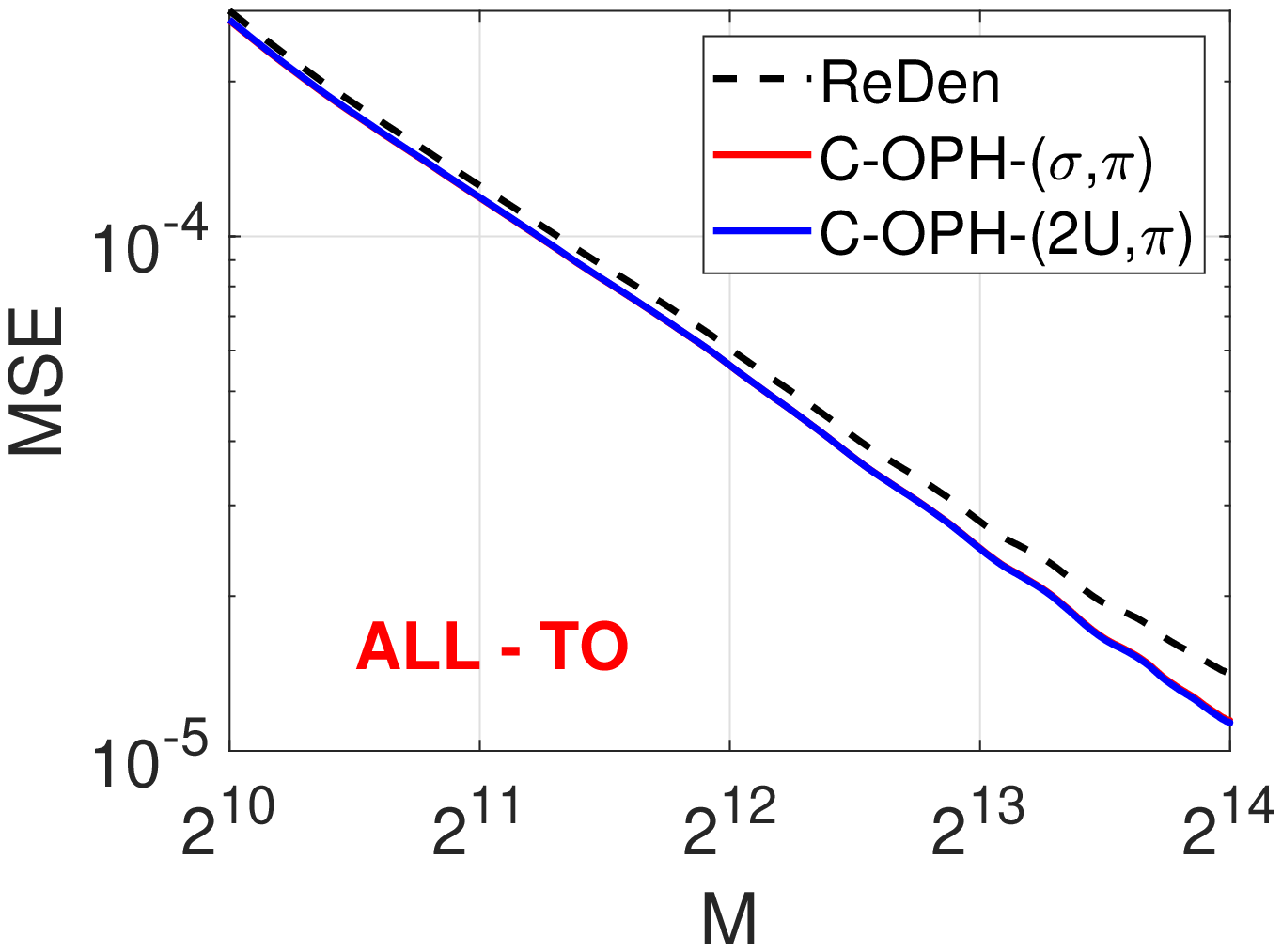}
		}
		\mbox{\hspace{-0.1in}
		\includegraphics[width=2.25in]{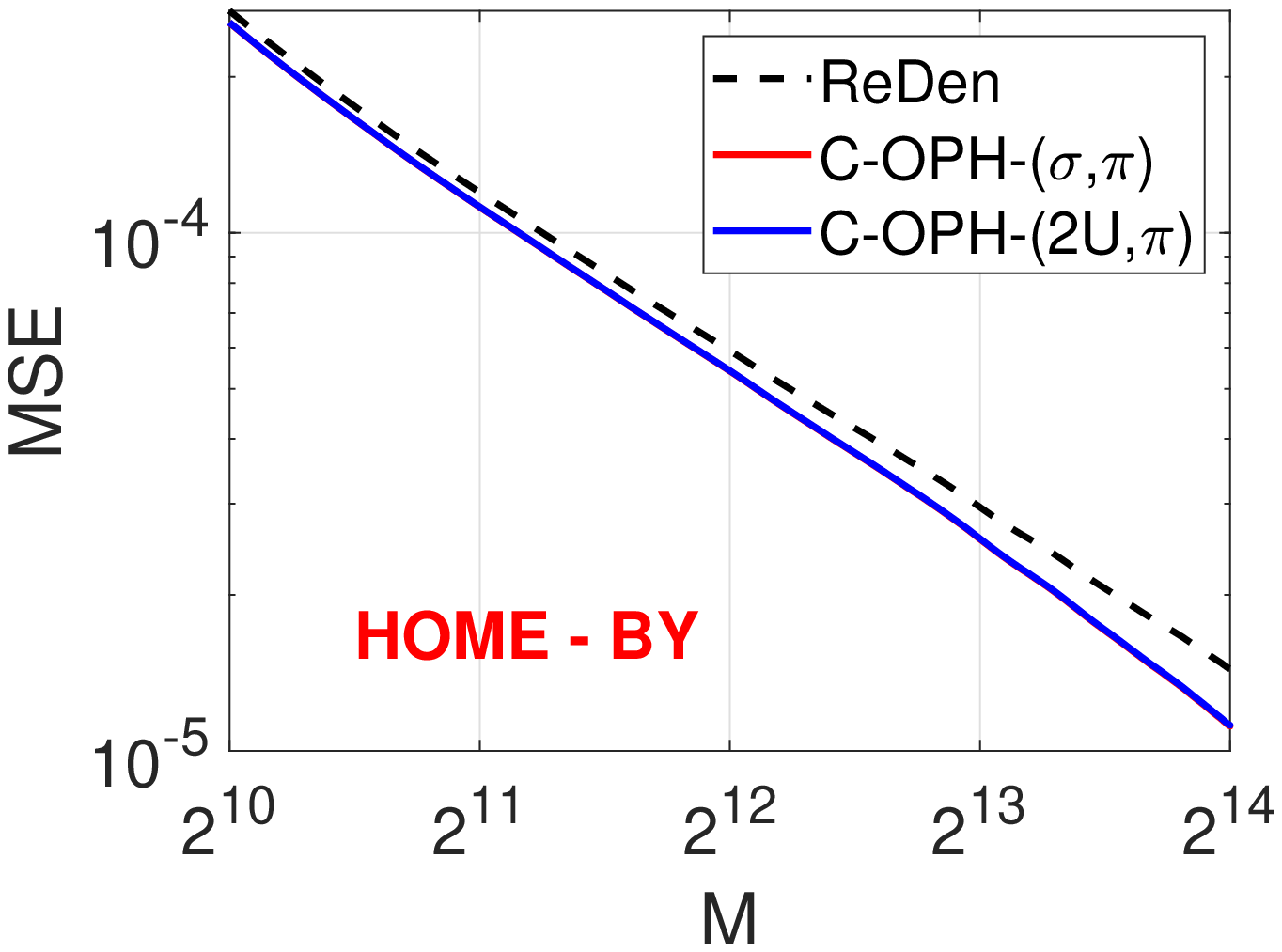}\hspace{-0.1in}
		\includegraphics[width=2.25in]{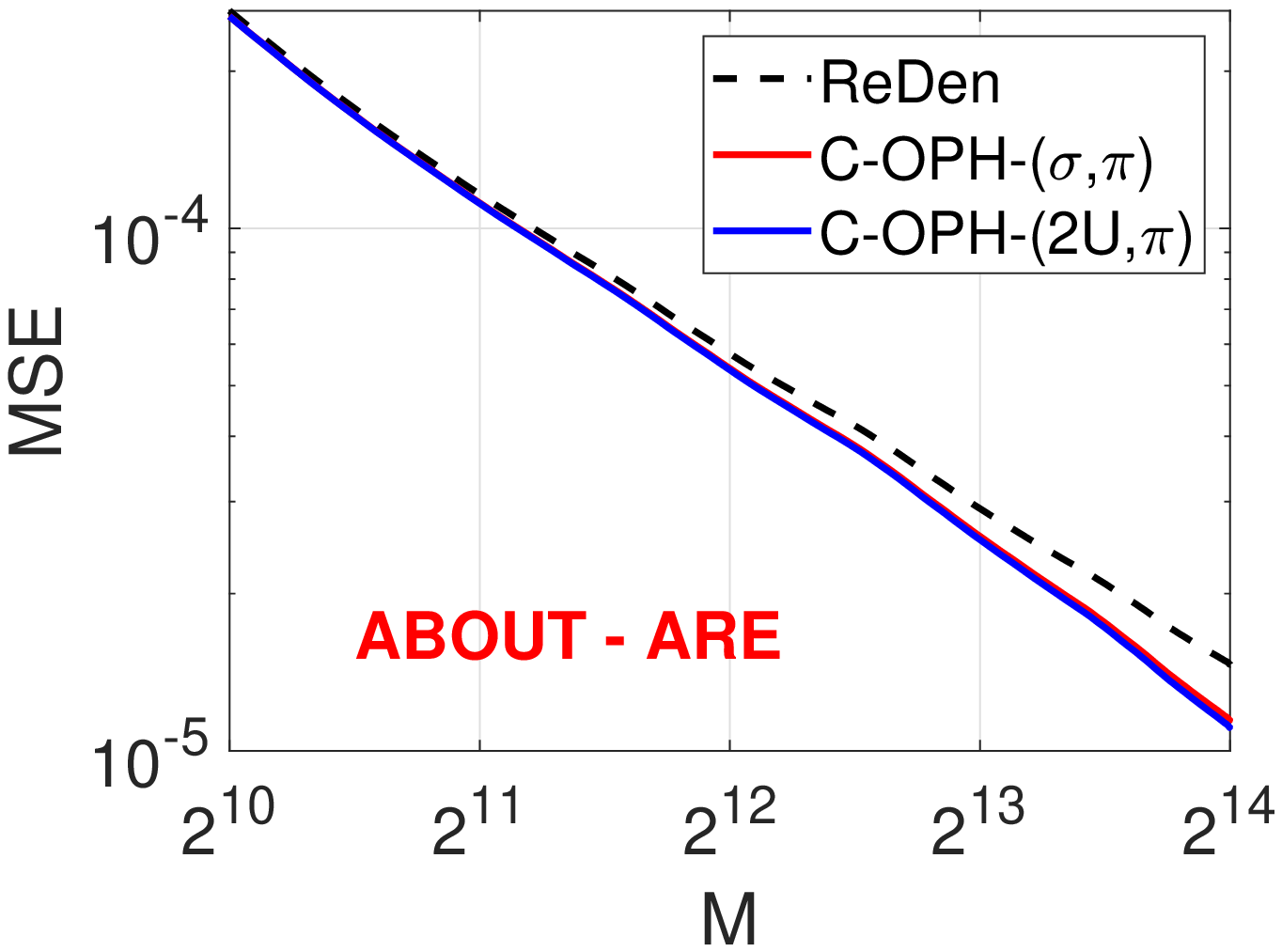}\hspace{-0.1in}
		\includegraphics[width=2.25in]{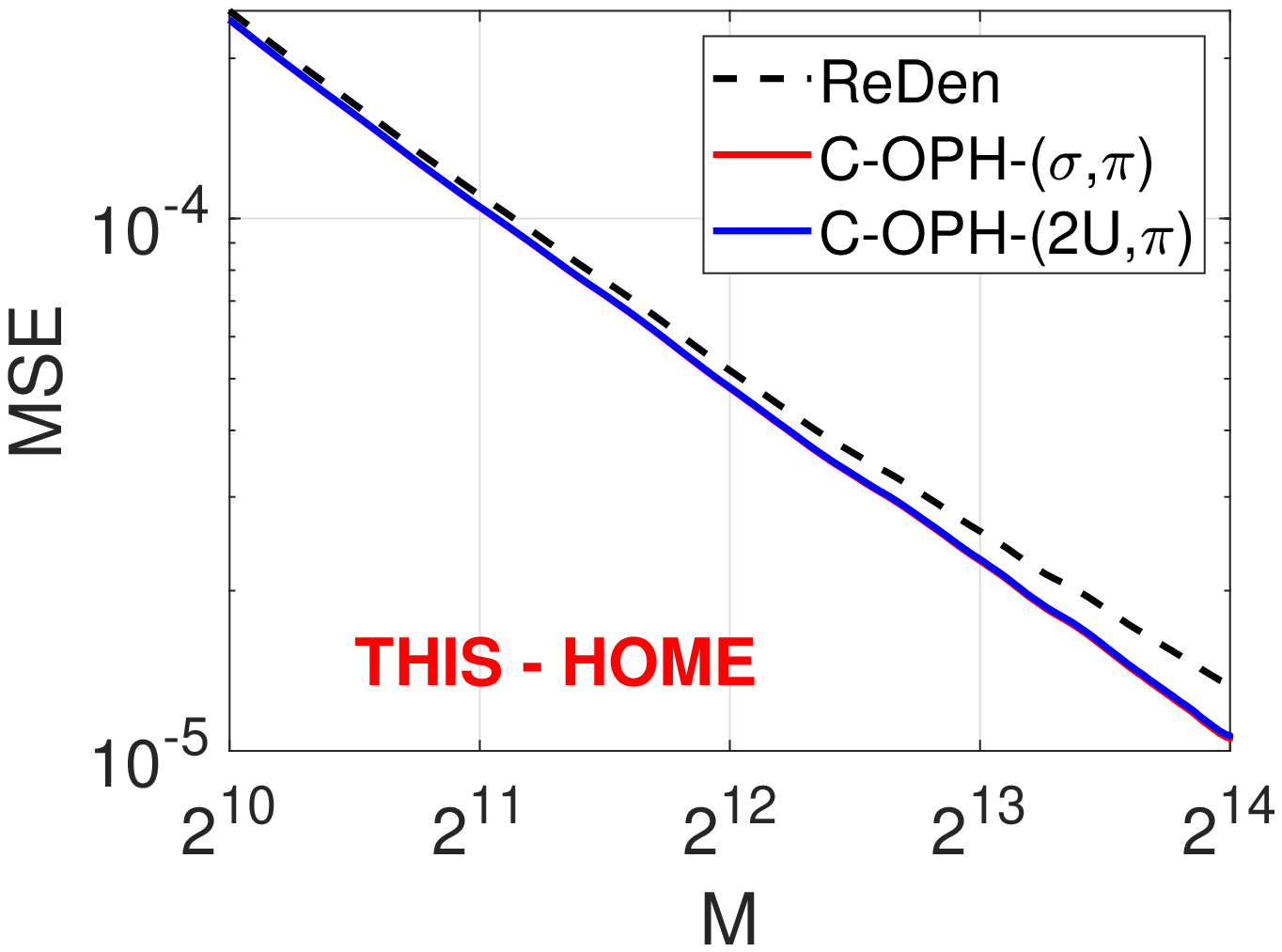}
		}
	\end{center}
	\vspace{-0.1in}
	\caption{Mean Squared Error (MSE) of ReDen, C-OPH-$(\sigma,\pi)$ and C-OPH-$(2U,\pi)$ on word pairs from the \textit{Words} dataset, $K=2^9$. We see that C-OPH improves the MSE of ReDen, and using hash function to perform bin split empirically gives same MSE as using perfectly random permutation.}
	\label{fig:MSE COPH words K512}
\end{figure}
\section{Conclusion}

The popular minwise hashing (MinHash) method has been improved by the so-called Circulant MinHash (C-MinHash)~\citep{CMH2Perm2021,CMH1Perm2021}. To generate $K$ hash values, C-MinHash only needs two permutations or even just one permutation, instead of using $K$ independent permutations as required by the standard MinHash. It is clear that C-MinHash is different from the previous known work on ``One Permutation Hashing''  (OPH)~\citep{Proc:Li_Owen_Zhang_NIPS12} and its variants (due to different ``densification'' schemes).  In this paper, by incorporating the central ideas of circulant permutations, we propose Circulant OPH (C-OPH) to improve the accuracy of the state-of-the-art densified OPH, i.e., the ``ReDen'' method developed in~\citet{Proc:Li_NIPS19_BCWS}.

\vspace{0.15in}

\noindent The basic idea of Circulant OPH (C-OPH) is to first randomly divide the data vectors into equal-sized $K$ bins, then use a smaller permutation (of size $D/K$, where $D$ is the data dimension) to generate the required number of hash values in a circulant  manner. We show that the proposed C-OPH achieves a smaller estimation variance  than ``ReDen'', the previous best (most accurate) densified OPH algorithm.  In addition to achieving improved estimation accuracy, another interesting benefit of C-OPH is that, practically speaking, C-OPH just needs ``$1/K$'' permutation instead of one permutation. This consequence would be useful in practice. For example, consider a dataset with $D = 2^{40}$ and  $K = 2^{10}$. Then we just need a small permutation of size $D/K = 2^{30}$ for this high-dimensional dataset, and $2^{30}$ is small  even for the GPU memory.

\bibliography{standard}
\bibliographystyle{plainnat}

\end{document}